\documentclass{article}
    
    \PassOptionsToPackage{numbers, compress}{natbib}
    
    \usepackage[preprint]{neurips_2026}

    \usepackage[utf8]{inputenc} 
    \usepackage[T1]{fontenc}    
    \usepackage{hyperref}       
    \usepackage{url}            
    \usepackage{booktabs}       
    \usepackage{amsfonts}       
    \usepackage{nicefrac}       
    \usepackage{microtype}      
    \usepackage{xcolor}         
    
    \usepackage[most]{tcolorbox}
    
    \usepackage{amsmath}
    \usepackage{amsthm}
    
    \usepackage{tikz}
    \usepackage{subcaption}
    \usetikzlibrary{decorations.pathreplacing, arrows.meta}
    
    \theoremstyle{plain}
    \newtheorem{theorem}{Theorem}[section]
    \newtheorem{lemma}[theorem]{Lemma}
    \newtheorem{corollary}[theorem]{Corollary}
    
    \theoremstyle{definition}
    \newtheorem{definition}{Definition}[section]
    
    \theoremstyle{remark}
    \newtheorem{remark}{Remark}[section]
    
    \newtheorem{assumption}{Assumption}

    \newtcolorbox{mainthmbox}{
      enhanced,
      breakable,
      colback=blue!8,
      colframe=blue!35!black,
      boxrule=0.4pt,
      arc=2pt,
      left=5pt,
      right=5pt,
      top=5pt,
      bottom=5pt
    }
    
    \newtcolorbox{definitionbox}{
      enhanced,
      breakable,
      colback=pink!25,
      colframe=pink!60!black,
      boxrule=0.4pt,
      arc=2pt,
      left=5pt,
      right=5pt,
      top=5pt,
      bottom=5pt
    }

    \setcitestyle{authoryear,round}
    
    \title{Effective Context in Transformers: An Analysis of Fragmentation and Tokenization}

    \author{%
      Amirmehdi J. Fesharaki \\
      Department of Communications and Electronics\\
      Institut Polytechnique de Paris\\
      \texttt{amirmehdi.jafari@ip-paris.fr}
      \And 
      Mohammadamin Rami \\
      Department of Communications and Electronics\\
      Institut Polytechnique de Paris\\
      \texttt{mohammadamin.rami@ip-paris.fr}
      \And
      Aslan Tchamkerten \\
      Department of Communications and Electronics\\
      Institut Polytechnique de Paris\\
      \texttt{aslan.tchamkerten@telecom-paris.fr}}
    
\begin{document}
    
    \maketitle

\begin{abstract}
Transformers predict over a representation of a sequence. The same data can
be written as bytes, characters, or subword tokens, and these representations
may be lossless. Yet, under a fixed context window, they need not expose the
same information to the model. This raises a basic question: how does the
choice of representation change what a finite-context predictor can achieve?

We study this question on Markov sources and uncover two complementary
phenomena. First, we observe that moving to smaller representation units can
hurt prediction even when the context window is enlarged to cover the relevant
source history. To explain this, we introduce fragmentation: a lossless
recoding that replaces each source symbol by several smaller units. We prove
that fragmentation can strictly increase the optimal finite-context log-loss,
showing that the gap is not merely an optimization or capacity issue, but can
be intrinsic to the representation. This gives a theoretical account of the
finite-context gap observed in byte- and character-level models such as ByT5
and CANINE relative to subword-tokenized models. Second, we study the
opposite direction: greedy tokenization---BPE, WordPiece, and related
methods---which groups source symbols into larger units. We show that
tokenization can make a short token window behave like a longer source-context
window, and we give a loss guarantee describing when this is achievable. The
guarantee depends on how reliably token windows span the needed source
history, together with the compression rate of the tokenizer. This also yields
a simple diagnostic for real tokenizers: measuring how much source context a
fixed token window reliably contains. Together, the two directions establish a finite-context information-theoretic framework for reasoning about representation choices in Transformers.
\end{abstract}

    \section{Introduction}
    Sequence prediction is at the core of modern language modeling, and the
    Transformer architecture has proven remarkably effective for this task
    \citep{vaswaniAttentionAllYou,devlinBERTPretrainingDeep2019,
    brownLanguageModelsAre,touvronLLaMAOpenEfficient2023}. 
    Most analyses of these models focus on architecture, optimization, scaling, or
    data. Yet there is a more basic choice made before training begins: how the
    sequence is represented. The same text can be represented as bytes, characters,
    words, or subword tokens. These encodings may be lossless, but a fixed model
    window over them can cover very different amounts of the original sequence. Thus,
    even lossless representations can lead to different finite-context prediction
    problems.

    This tension appears directly in modern language modeling. Token-free models
    such as ByT5 and CANINE avoid a fixed subword tokenizer by operating on bytes or
    characters, but they must process longer sequences and face a more expensive
    context problem~\citep{xueByT5TokenFreeFuture2022,
    clarkCaninePretrainingEfficient2022}. Other approaches reduce this cost using
    latent subword blocks or multiscale byte-level modeling~\citep{
    tayCharformerFastCharacter2021,yuMEGABYTEPredictingMillionbyte2023}. At the
    same time, most mainstream language models still use a fixed discrete tokenizer
    learned before model training, often based on subword or byte-level variants of
    BPE, WordPiece, or unigram language-model tokenization
    \citep{sennrichNeuralMachineTranslation2016,
    schusterJapaneseKoreanVoice2012,devlinBERTPretrainingDeep2019, kudoSentencePieceSimpleLanguage2018a}. This creates a basic representation tradeoff: token-free models avoid committing
    to a fixed tokenizer, but they operate on longer sequences; tokenized models
    shorten the sequence, but they introduce a preprocessing choice which changes the prediction problem.
    
    To study this tradeoff cleanly, we focus on finite-context prediction on
    stationary finite-order Markov sources. In this setting, the relevant history is
    explicit: once the context reaches the Markov order, the optimal finite-context
    loss reaches the entropy rate. This lets us isolate how the representation
    changes what is available inside a fixed window. We ask:
    
    
    \begin{center}
    \textbf{How does representation shape what a finite-context predictor can achieve?}
    \end{center}
    
    We show that this question has two sides. Moving to smaller representation units
    can make prediction harder, even when the recoding is lossless. Moving to larger
    token units can make the same finite window carry more information about the
    original source. The paper formalizes both effects and identifies the quantities
    that determine when they arise. In short, our contributions are:
    
    \begin{itemize}
        \item We uncover a \emph{fragmentation effect} motivated by
    Figure~\ref{fig:fragmentation_transformer}: when a source is recoded into
    sub-symbol units, a transformer can converge to a strictly worse
    loss even after its window is enlarged to cover the relevant source history.
    This mirrors the prediction problem faced by byte- or character-level models.
    We formalize the effect as a lossless recoding into smaller units and prove that
    the resulting excess loss can be intrinsic to the representation.
    
    \item We develop a finite-context analysis of greedy tokenization. We show
        that if a token window spans enough source symbols, then a token-level
        predictor can match the performance of a longer source-context predictor
        with no asymptotic loss. We then extend this to typical token windows,
        giving an achievability bound controlled by the lower tail of the
        source-span distribution and the tokenizer compression rate.
    
    \item We identify the effective source context of a tokenizer as a
    measurable quantity for finite-context prediction. The theory lets us compare
    tokenizers by the guarantees they provide: for a fixed token window, which
    tokenizer can reliably match a source-level predictor with the largest possible
    context, up to small excess loss? This criterion depends on the lower tail of
    the token-window source-span distribution, not on average token length alone.
    \end{itemize}

    \subsection{Related Work}
    
    Beyond the standard tokenizer algorithms cited above, recent work studies
    tokenization as a modeling and efficiency choice.~\cite{goldmanUnpackingTokenizationEvaluating2024} connect tokenizer
    compression to downstream model performance, while Tao et
    al.~\citep{taoScalingLawsVocabulary} study how vocabulary size should scale
    with model size and compute. Other work highlights limitations of fixed subword
    tokenization, including robustness failures and sensitivity to token
    structure~\citep{chaiTokenizationFallingShort2024}, as well as broader effects
    on meaning, reasoning, and multilingual efficiency
    ~\citep{haslettTokenizationChangesMeaning2025,
    zhangTokenizationConstraintsLLMs2025}. Our results are complementary:
    rather than evaluating tokenizer choices empirically, we characterize how a
    lossless representation changes the best achievable finite-context loss.
    
    The closest theoretical line studies Transformers and related sequence models on
    Markov sources~\citep{makkuvaAttentionMarkovCurious2024,
    rajaramanAnalysisTokenizationTransformers2024,
    rajaramanTransformersMarkovData2024a,
    edelmanEvolutionStatisticalInduction2024,
    bondaschiMarkovLaplaceHow2025}. These works use Markov data as a controlled
    setting to reveal what sequence models learn in context. In particular, ~\citep{makkuvaAttentionMarkovCurious2024} show that single-layer
    Transformers can behave like unigram predictors on higher-order Markov chains,
    while ~\citep{rajaramanAnalysisTokenizationTransformers2024}
    show that tokenization can make a unigram-style predictor near-optimal under
    suitable learned dictionaries. Our work stays in the Transformer finite-context
    view, but shifts the question from model behavior to representation: how does a
    lossless recoding change the best loss achievable by any fixed-window predictor?
    
    Finally, our greedy tokenizer model is related to dictionary parsing and
    compression, including Lempel--Ziv and LZW methods
    ~\citep{zivUniversalAlgorithmSequential1977,
    zivCompressionIndividualSequences1978,welchTechniqueHighPerformanceData1984},
    and recent compression/coding views of tokenization
    ~\citep{zouharTokenizationNoiselessChannel2023,
    schmidtTokenizationMoreCompression2024}. Both settings trade vocabulary size for
    longer emitted units; we translate this tradeoff into an achievability guarantee
    based on the source-span distribution of token windows.
    
    \section{Background}
    \label{sec:background}
    Throughout the paper, $\{Y_i\}_{i\in\mathbb{Z}}$ is an ergodic stationary finite order-$k$ Markov process defined over a
    finite alphabet $\mathcal{Y}$.  Recall that a process $\{Y_i\}_{i\in\mathbb{Z}}$ is $k$th-order Markov if,
    for every $i$,
    $
        \Pr(Y_i \mid Y_{-\infty}^{i-1})
        =
        \Pr(Y_i \mid Y_{i-k}^{i-1})
        \quad \text{almost surely}
         \label{eq:kth-order-markov}
    $. We shall also encounter processes that are not stationary but retain a
    periodic form of stationarity.  A process $\{X_i\}_{i\in\mathbb{Z}}$ is
    \emph{cyclostationary with period $M$} if its finite-dimensional
    distributions are invariant under shifts by multiples of $M$:
    $
        (X_{i_1},\ldots,X_{i_m})
        \overset{d}{=}
        (X_{i_1+M},\ldots,X_{i_m+M})
        \label{eq:cyclostationarity}
    $
    for all $m$ and all indices $i_1,\ldots,i_m$.  Such processes arise when
    each symbol of a stationary source is represented by a block of $M$
    subsymbols.
    
    \subsection{Finite-Context Predictors}
    We study predictors that, at time $i$,
    observe only a finite suffix of the past and assign a probability to the
    next symbol. For an integer $w\geq 0$, a \emph{$w$-context predictor} is a collection
    of probability distributions on $\mathcal{Y}$, one for each context
    $c\in\mathcal{Y}^w$.  Equivalently, it is a map
    $
        q:\mathcal{Y}\times\mathcal{Y}^w \to [0,1], \; \text{satisfying} \;\sum_{y\in\mathcal{Y}} q(y\mid c)=1,
       \; c\in\mathcal{Y}^w .
    $
    We denote the class of all such predictors by
    $\mathcal{Q}_w^{\mathcal{Y}}$. Given $q\in\mathcal{Q}_w^{\mathcal{Y}}$, its population log-loss on the
    source $\{Y_i\}$ is
    $
        L_Y(q)
        :=
        \mathbb{E}
        \left[
            -\log_2 q\!\left(Y_0 \mid Y_{-w}^{-1}\right)
        \right].
        \label{eq:population-log-loss}
    $ The minimum loss with context length $w$ is then
    \begin{equation}
        \mathcal{L}_Y(w)
        :=
        \min_{q\in\mathcal{Q}_w^{\mathcal{Y}}} L_Y(q).
        \label{eq:optimal-population-loss}
    \end{equation}
    For completeness, recall that if $q$ has finite expected log-loss, then Birkhoff’s
    ergodic theorem (see, {\it{e.g.}}, \cite{coverElementsInformationTheory2005}) yields
    $
        \frac{1}{n}
        \ell(q;Y^n)
        \overset{n\to \infty}{\longrightarrow}
        L_Y(q)
    \; \text{a.s.,}
        \label{eq:empirical-to-population-loss}
        \; \text{where} 
        \;\ell(q;Y^n) :=  
        \sum_{i=w+1}^{n}
            -\log_2 q\!\left(Y_i \mid Y_{i-w}^{i-1}\right),
    $
    {\it i.e.}, the population loss is the almost-sure asymptotic normalized
    log-loss along ``typical''  realizations.
    
    Specializing to our Markovian source, the optimal $w$-context loss is
    $
        \mathcal{L}_Y(w)
        =
        H(Y_0\mid Y_{-w}^{-1}),
        \label{eq:optimal-loss-entropy}
    $
    which is a non-increasing function of $w$---since conditioning reduces entropy.
    In particular, once the context length reaches the Markov order, {\it{i.e.}}, $w \ge k$, the minimum loss becomes 
    $
        \mathcal{L}_Y(w)
        =
        H(Y_0\mid Y_{-k}^{-1}).
        \label{eq:markov-loss-saturation}
    $
    
    \paragraph{Notation.} All logarithms are base $2$, unless specified otherwise. For any alphabet $\mathcal{Y}$, the set $\mathcal{Y}^*$ is defined as $\mathcal{Y}^* = \bigcup_{n \ge 0} \mathcal{Y}^n$.

    \section{Can Lossless Recoding Hurt Finite-Context Prediction?}\label{sec:fragmentation}
    We first examine the effect of refining the representation.  We call this
    operation \emph{fragmentation}: each source symbol is replaced, losslessly,
    by a short string over a smaller alphabet.  This idealizes the passage from
    words or characters to bytes or bits.
    
    Fragmentation preserves the source sequence, but not necessarily the
    information available at finite context.  A finite-context predictor sees
    only a fixed window in the chosen representation.  After fragmentation, this
    window may cut through source-symbol boundaries, obscure the phase of the
    next fragment within its source symbol, and fail to contain the source
    history available before fragmentation.  Thus a globally lossless refinement
    can nevertheless increase finite-context prediction loss.
    
    \paragraph{Fragmentation can hurt even with sufficient context.}
    Figure~\ref{fig:fragmentation_transformer} gives a first indication of the
    effect.  We generate a stationary Markov source of order $k$ and compare two
    predictors with the same architecture.  One predicts the source symbols
    directly from a context of length $w$.  The other predicts a fragmented
    representation: each source symbol is replaced by a block of $M$ smaller
    symbols, and the context length is enlarged to $Mw$.
    
    Note that we choose $w>k$.  Hence the source-level predictor sees the entire Markov
    state.  The fragmented predictor, in turn, sees $Mw$ fragments, the nominal
    span of $w$ source symbols.  On this accounting, the two predictors should
    have access to the same amount of relevant history, and one might expect the
    same limiting loss, measured in bits per source symbol.
    
    They do not.  The fragmented predictor converges to a strictly larger loss.
    The gap persists after training has stabilized, and therefore is not merely
    an optimization transient or an artifact of insufficient capacity.  It is a
    property of the representation.  Fragmentation changes the finite-context
    prediction problem: even when the window is long enough in source-symbol
    units, it may not reveal the phase of the next fragment nor align with the
    source history used by the unfragmented predictor.  The remainder of this
    section explains this gap.
    \begin{figure}[tbp]
    \centering
    \begin{minipage}[t]{0.49\linewidth}
        \centering
        \includegraphics[width=\linewidth]{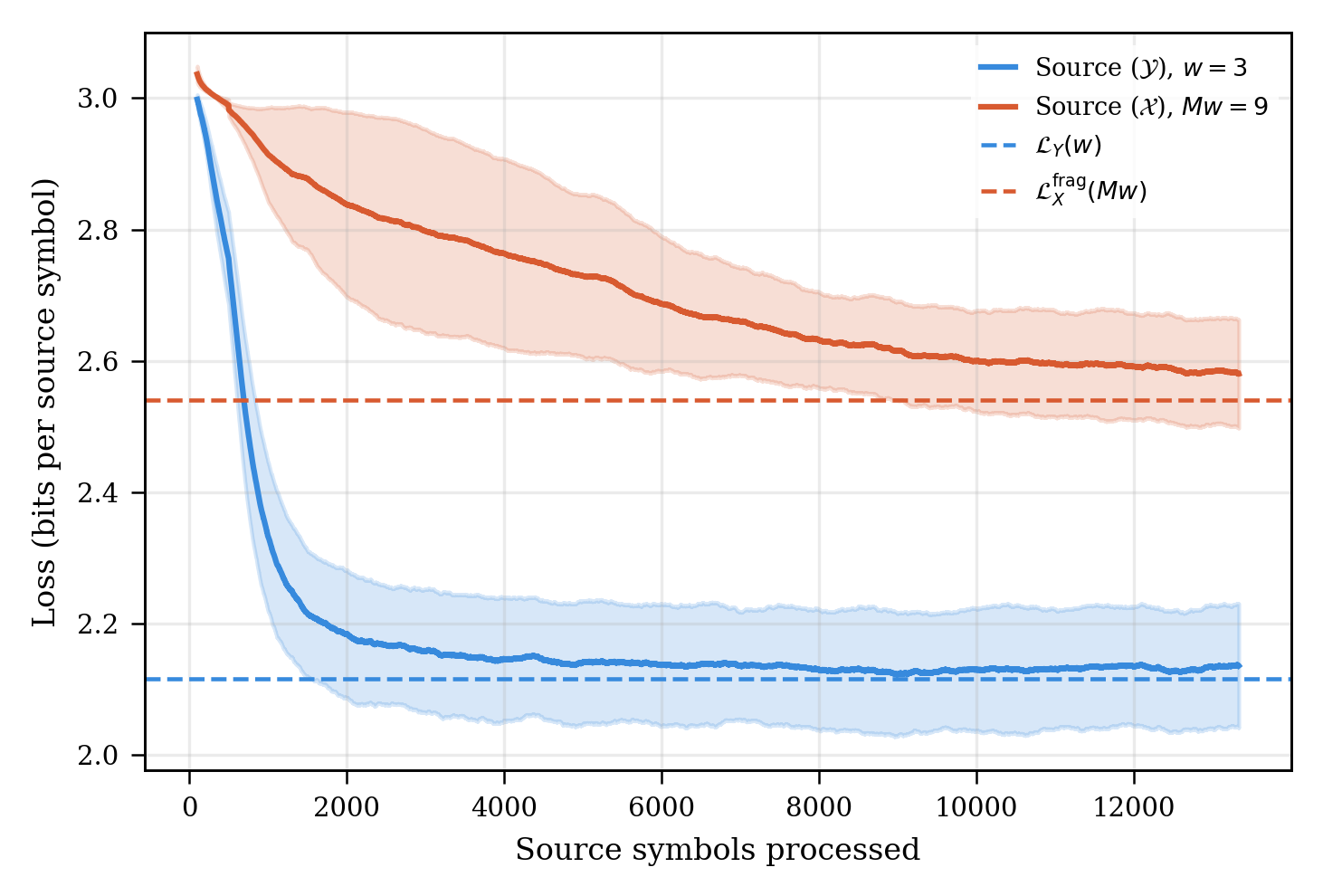}
        \captionof{figure}{
    Fragmentation creates a persistent finite-context loss gap. 
    A source-level Transformer with $w>k$ reaches the Markov optimum, while the
    fragmented model with context $Mw$ converges to a higher loss, despite having
    the same nominal source span.
    }
        \label{fig:fragmentation_transformer}
    \end{minipage}
    \hfill
    \begin{minipage}[t]{0.49\linewidth}
    \centering
    \raisebox{0.5cm}{%
    \begin{tikzpicture}[
        xscale=0.95,
        ysymbol/.style={font=\small, minimum width=1.1cm, minimum height=0.45cm},
        xbit/.style={font=\scriptsize\ttfamily, minimum width=0.5cm, minimum height=0.4cm, draw, thick},
        bracket/.style={thick, decorate, decoration={brace, amplitude=3pt, mirror}},
        label/.style={font=\scriptsize}
    ]
        \node[label, anchor=west] at (-0.5, 0.9) {\textit{Case 1:}};
        \node[ysymbol] at (0.5, 0.55) {$d$};
        \node[ysymbol] at (1.5, 0.55) {$b$};
        \node[ysymbol] at (2.7, 0.55) {$X_j = ?$};
    
        \node[xbit] at (0.25, 0.0) {1};
        \node[xbit, fill=blue!15] at (0.75, 0.0) {1};
        \node[xbit, fill=blue!15] at (1.25, 0.0) {0};
        \node[xbit, fill=blue!15] at (1.75, 0.0) {1};
        \node[xbit] at (2.25, 0.0) {?};
        \node[xbit] at (2.75, 0.0) {?};
    
        \draw[dashed, gray, thick] (0.99, 0.72) -- (0.99, 0.22);
        \draw[dashed, gray, thick] (1.99, 0.72) -- (1.99, 0.22);
        \draw[bracket] (0.5, -0.22) -- (2.0, -0.22)
            node[midway, below=2pt, label] {ctx $= 101$};
        \draw[->, thick, red] (2.25, 0.35) -- (2.25, 0.2);
    
        \node[label, anchor=west] at (-0.5, -1.1) {\textit{Case 2:}};
        \node[ysymbol] at (0.5, -1.45) {$a$};
        \node[ysymbol] at (1.5, -1.45) {$c$};
        \node[ysymbol] at (2.7, -1.45) {$X_j = ?$};
    
        \node[xbit] at (0.25, -2.0) {0};
        \node[xbit] at (0.75, -2.0) {0};
        \node[xbit, fill=blue!15] at (1.25, -2.0) {1};
        \node[xbit, fill=blue!15] at (1.75, -2.0) {0};
        \node[xbit, fill=blue!15] at (2.25, -2.0) {1};
        \node[xbit] at (2.75, -2.0) {?};
    
        \draw[dashed, gray, thick] (0.99, -1.28) -- (0.99, -1.78);
        \draw[dashed, gray, thick] (1.99, -1.28) -- (1.99, -1.78);
        \draw[bracket] (1.0, -2.22) -- (2.5, -2.22)
            node[midway, below=2pt, label] {ctx $= 101$};
        \draw[->, thick, red] (2.75, -1.65) -- (2.75, -1.8);
    \end{tikzpicture}%
    }
    \captionof{figure}{A stationary first-order Markov source on $\mathcal{Y} = \{a,b,c,d\}$ with the bijection $\varphi(a) = 00,\ \varphi(b) = 01,\ \varphi(c) = 10,\ \varphi(d) = 11$. A bit-level predictor with context $w = 3$ observing $101$ cannot distinguish whether $X_j$ can be any symbol in $\mathcal{Y}$ (Case~$1$) or $X_j\in\{c,d\}$ (Case~$2$).}
    \label{fig:motivating_example}
    \end{minipage}
    \end{figure}
    We first formalize fragmentation:
    
    \begin{definitionbox}
    \begin{definition}[Fragmentation]
    Let $M>1$ be an integer and let $\mathcal{X}$ be a finite alphabet such
    that
    $
        |\mathcal{Y}| \leq |\mathcal{X}|^M .
    $
    A \emph{fragmentation map} is an injective map
    \[
        \varphi:\mathcal{Y}\to\mathcal{X}^M .
    \]
    Given a process $\{Y_i\}_{i\in\mathbb{Z}}$ over $\mathcal{Y}$, its
    fragmented process $\{X_j\}_{j\in\mathbb{Z}}$ over $\mathcal{X}$ is
    obtained by replacing each source symbol $Y_i$ by the length-$M$ block $\varphi(Y_i)$.  That is, for every
    $i\in\mathbb{Z}$ and $\theta\in\{1,\ldots,M\}$,
    $$
        X_{M(i-1)+\theta}
        =
        \varphi(Y_i)_\theta .
    $$
    
    We call $\theta$ the \emph{phase} of the fragmented symbol
    $X_{M(i-1)+\theta}$.
    \end{definition}
    \end{definitionbox}
    In what follows, context lengths for the fragmented process are measured in
    $\mathcal{X}$-symbols, {\it{i.e.}}, in fragments.  Thus a window of length $Mw$
    in the fragmented process has the same nominal span as $w$ source symbols.
    
    \paragraph{Example.}
    Let $\mathcal{Y}=\{a,b,c,d\}$, and fragment each source symbol into two
    bits by
    $
        \varphi(a)=00,\qquad
        \varphi(b)=01,\qquad
        \varphi(c)=10,\qquad
        \varphi(d)=11 .
    $
    Each bit of the fragmented process has one of two phases: it is either
    the first or the second bit of the encoded source symbol; see
    Figure~\ref{fig:motivating_example}. Suppose that $\{Y_i\}$ is first-order Markov, and consider a bit predictor
    with context length $3$.  This context always contains a complete
    preceding two-bit block.  Yet it need not determine the phase of the bit
    to be predicted.  Figure~\ref{fig:motivating_example} shows two
    occurrences of the same bit context, $101$.
    
    
    In Case~$1$, the context ends at a block boundary, so the next bit is the
first bit of a new source symbol, distributed according to the Markov transition
from the previous symbol. In Case~$2$, the context ends inside a block, so the
next bit completes the current source symbol; after observing first bit $1$, the
symbol is restricted to $c$ or $d$. The predictor, however, only observes the context string $101$ and not the
phase at which it occurs. It must therefore assign the same next-bit
distribution in both cases, even though the correct conditional distributions
may differ. This phase mismatch creates the excess loss quantified below.

    To state the gap, we put both losses on the same scale: bits per original
    source symbol.  Let $\mathcal{Q}_{w}^{\mathcal{X}}$ be the class of
    $w$-context predictors on the fragmented alphabet $\mathcal{X}$, where
    $w$ is measured in fragments.  For a source block $Y_1^n$, write
    \[
        \varphi(Y_1^n)
        :=
        \varphi(Y_1)\varphi(Y_2)\cdots\varphi(Y_n)
        \in \mathcal{X}^{Mn}
    \]
    for the corresponding fragmented string. For $q\in\mathcal{Q}_{w}^{\mathcal{X}}$, define its fragmented population
    loss by
   $
        L_X^{\mathrm{frag}}(q)
        :=
        \sum_{\theta=1}^{M}
        \mathbb{E}\!\left[
            -\log_2
            q\!\left(
                X_{\theta}
                \mid
                X_{\theta-w}^{\theta-1}
            \right)
        \right].
        \label{eq:fragmented-population-loss}$
    The expectation in the $\theta$th term is taken at phase $\theta$.  By
    stationarity of $\{Y_i\}$, the fragmented process is cyclostationary with
    period $M$, so the distribution of each phase-$\theta$ term is independent
    of the source-symbol index.  Thus $L_X^{\mathrm{frag}}(q)$ is the expected
    cumulative log-loss incurred in predicting the $M$ fragments of one source
    symbol.\footnote{If
    $L_X^{\mathrm{frag}}(q)<\infty$, then Birkhoff's theorem applied to the
    stationary block process yields
    $
        n^{-1}\ell(q;\varphi(Y_1^n))
        \to
        L_X^{\mathrm{frag}}(q)
    $
    almost surely.} The optimal fragmented loss with context length $w$ is
    \begin{equation}
        \mathcal{L}_X^{\mathrm{frag}}(w)
        :=
        \min_{q\in\mathcal{Q}_{w}^{\mathcal{X}}}
        L_X^{\mathrm{frag}}(q).
        \label{eq:optimal-fragmented-loss}
    \end{equation}
    We compare $\mathcal{L}_X^{\mathrm{frag}}(Mw)$ with $\mathcal{L}_Y(w)$ because
$Mw$ fragments have the same nominal span as $w$ source symbols, and both losses
are measured in bits per source symbol. A source predictor assigns probability
directly to $Y_i$, while a fragmented predictor assigns probability to the same
symbol through its $M$ fragments. The comparison is nontrivial because the
fragmented window need not align with source-symbol boundaries. The next result
quantifies the resulting excess loss; the proof is given in
Appendix~\ref{proof:fragmentation}.

    \begin{mainthmbox}
    \begin{theorem}[Fragmentation]\label{thm:fragmentation}
    Let $\{Y_i\}$ be a stationary $k^{\text{th}}$-order Markov process on $\mathcal{Y}$, and let $\{X_j\}$ be its fragmented process under $\varphi$. For any context window size $w$,
    \begin{align}\label{eq:fragmentation-theorem}
    \mathcal{L}_X^{\mathrm{frag}}(Mw) - \mathcal{L}_Y(w)
    \;=\; \underbrace{\sum_{\theta = 2}^{M} I\!\left(X_{\theta}; X_{1-Mw}^{\theta-Mw - 1} \,\big|\, X_{\theta-Mw}^{\theta-1}\right)}_{\text{context deficit}} \nonumber  + \underbrace{M \cdot I\!\left(\Theta \,;\, X_{\Theta} \,\big|\, X_{\Theta-Mw}^{\Theta-1}\right)}_{\text{phase ambiguity}},
    \end{align}
   where $\Theta$ is an auxiliary random variable, independent of $\{X_j\}$ and
uniform on $\{1,\ldots,M\}$, representing the phase of the prediction target.
The context-deficit term vanishes when $w>k$.
    \end{theorem}
    \end{mainthmbox}
    Since both terms are nonnegative,
    $
        \mathcal{L}_X^{\mathrm{frag}}(Mw)
        \ge
        \mathcal{L}_Y(w).
    $
    Thus, a lossless finer representation can strictly increase the optimal
    finite-context prediction loss.
    
    \paragraph{Proof idea.}
A fragmented predictor conditions only on the last $Mw$ fragments.  Hence
two occurrences with the same fragment context are treated identically, even
if they occur at different phases.  Equivalently, its optimal loss is the
loss of a prediction problem in which the phase of the prediction time is
first randomized and then hidden from the predictor.  Comparing this
phase-hidden loss with the chain-rule expansion of the source-level loss
separates the gap into two terms: the information needed to identify the
phase, and the source history lost because the fragmented window is not
aligned with source-symbol boundaries.
    
    \paragraph{Implication for byte- and character-level models.}
Theorem~\ref{thm:fragmentation} gives a possible explanation for the
finite-context gap between byte- or character-level models, such as
ByT5~\citep{xueByT5TokenFreeFuture2022} and
CANINE~\citep{clarkCaninePretrainingEfficient2022}, and tokenized models.
Such models represent a source unit as a sequence of symbols in a finer
alphabet.  Even when the model is well trained, and even when its window has
the nominal span needed to cover the relevant source history, the local
prediction problem may remain harder.  The same fragment context can arise
at different positions inside the original source unit, and these positions
can induce different conditional distributions for the next fragment.  The
resulting loss is therefore not merely an optimization or capacity effect;
it can be intrinsic to the representation.

There is a second cost.  Finer representations lengthen the sequence.  To
cover the same amount of source history, a byte- or character-level model
needs a longer context window than a token-level model.  For Transformers,
this increase is expensive.  Token-free models must therefore either pay the
longer-sequence cost directly~\citep{xueByT5TokenFreeFuture2022}, or reduce
it through architectural mechanisms such as
downsampling~\citep{clarkCaninePretrainingEfficient2022}.  This points to
the complementary question studied next: when source symbols are grouped
into tokens, what property of the tokenizer determines how much source
history a fixed token window contains?
    
  \section{When Does Tokenization Extend Effective Context?}
\label{sec:tokenization}

We now consider the reverse transformation.  Tokenization groups consecutive
source symbols into units from a larger alphabet.  Its advantage is that a
fixed number of tokens may cover more source history than the same number of
source symbols, thereby reducing finite-context loss.

However, this gain is not automatic.  A token predictor conditions on the output of
an encoder, and under greedy tokenization the relevant token boundaries may
depend on symbols not yet predicted.  Thus, token history is not merely a
reblocking of source history.  We need conditions under which a short token
context contains enough determined history to emulate a longer source-context
predictor.
    
  \paragraph{Empirical observation.}
Figure~\ref{fig:transformer_results} shows the phenomenon studied in this
section.  We use a binary stationary Markov source of order $k=12$.  On the
original sequence, a Transformer needs a window of length $12$ to reach the
entropy rate.  After tokenization, the same architecture can reach the entropy
rate with a shorter token window, despite predicting over a larger alphabet.
As the token vocabulary grows, the limiting loss decreases and approaches the
entropy rate.

This section explains when such a gain is possible.  The question is when a
short token context carries enough determined source history to match a
longer source-context predictor, and how much loss is incurred when it does
not.
    
    \begin{figure}[t]
    \centering
    \begin{subfigure}[t]{0.49\linewidth}
    \centering
    \includegraphics[width=\linewidth]{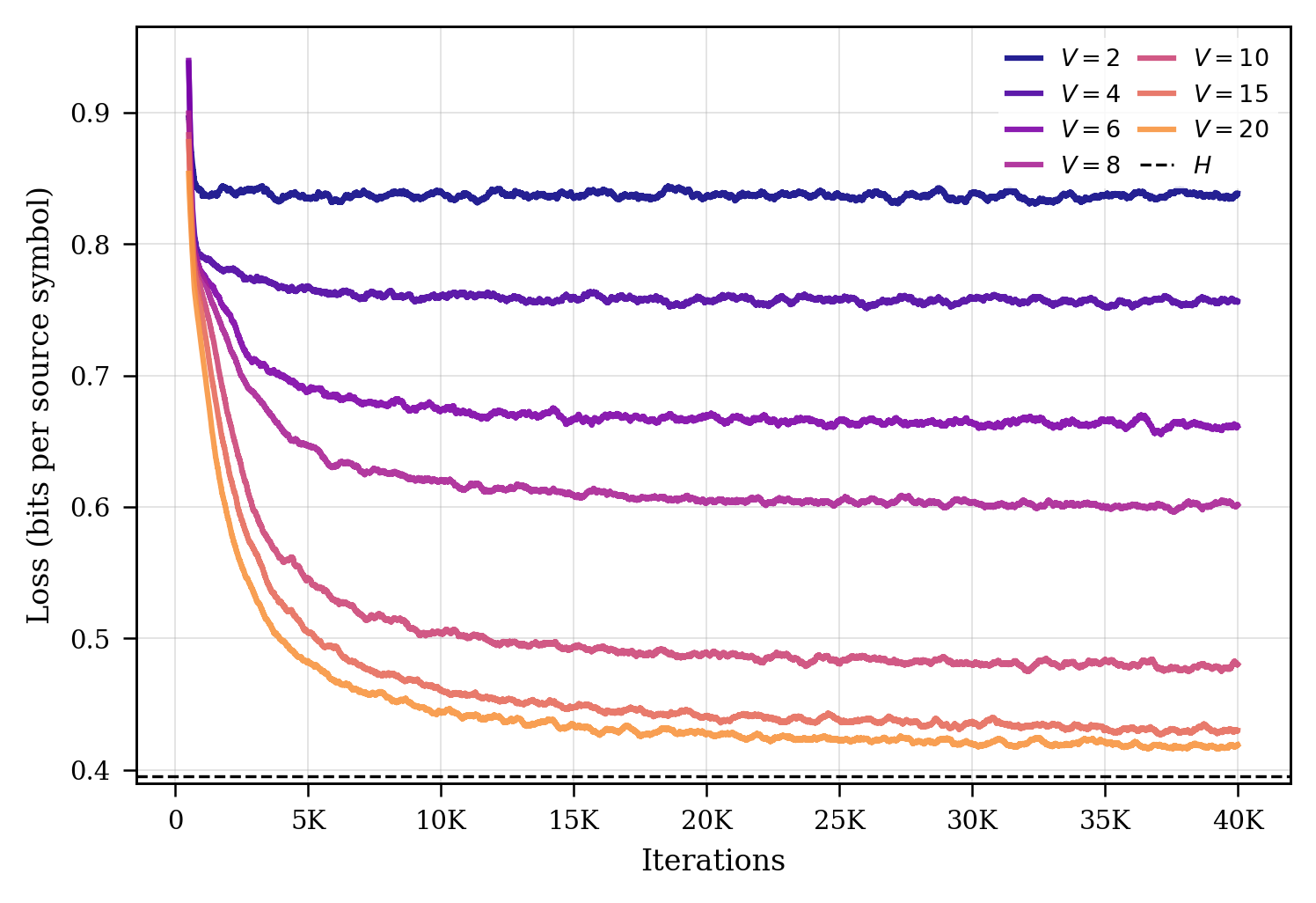}
    \caption{Training curves for transformers with $w=4$ and increasing vocabulary
    size on a binary order-$12$ Markov source.}
    \label{fig:curves_tokenization}
    \end{subfigure}
    \hfill
    \begin{subfigure}[t]{0.49\linewidth}
    \centering
    \includegraphics[width=\linewidth]{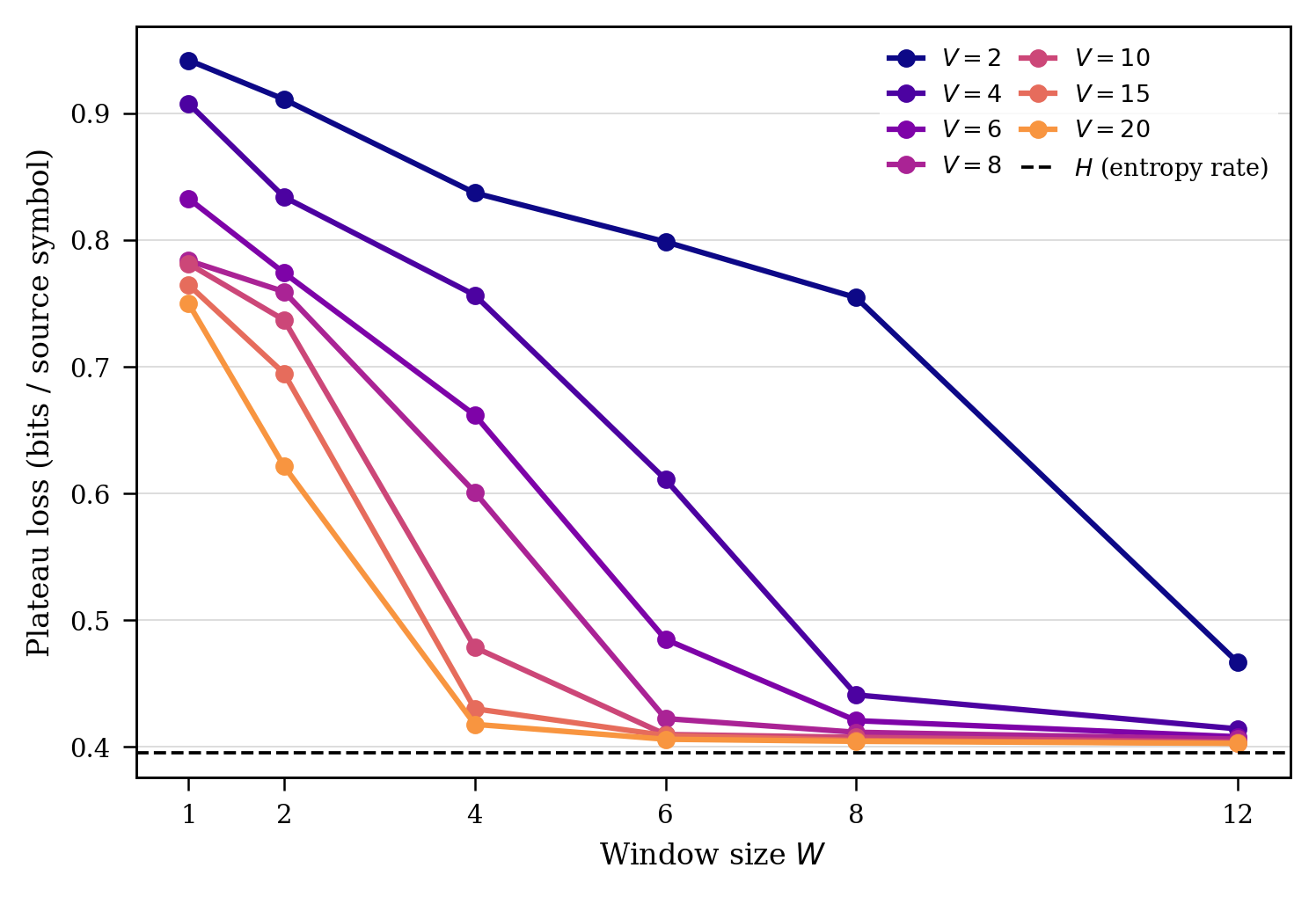}
    \caption{Plateau loss versus context window. Larger vocabularies reach the
    entropy rate with smaller token windows.}
    \label{fig:loss_w_tokenization}
    \end{subfigure}
    \caption{
    Tokenization extends effective context on a binary order-$12$ Markov source.
    The same architecture reaches the entropy rate with a much smaller token window
    after tokenization, consistent with the source-span mechanism formalized below.
    }
    \label{fig:transformer_results}
    \end{figure}
    
    \paragraph{Greedy tokenization.}
    We model a tokenizer as a prefix-closed finite vocabulary of source strings, organized as a
    prefix tree, together with a greedy parser. The vocabulary
    $\mathcal{Z}\subset \mathcal{Y}^*$ contains all single-symbol strings in
    $\mathcal{Y}$, namely $\mathcal{Y}\subset\mathcal{Z}$. Given a source sequence
    $y_1^n$, greedy tokenization scans from left to right and emits the longest
    vocabulary string matching the remaining source sequence. We write
    \[
        z_1^m = \Pi_{\mathcal{Z}}(y_1^n)
    \]
    for the resulting token sequence. We also define the expansion (decoding) map
    $\Gamma_{\mathcal{Z}}:\mathcal{Z}^*\to\mathcal{Y}^*$, which maps a token
    sequence back to the underlying source sequence by concatenating the source
    strings represented by its tokens. Thus,
    \[
        \Gamma_{\mathcal{Z}}(z_1^m)
        =
        \Gamma_{\mathcal{Z}}(z_1)\Gamma_{\mathcal{Z}}(z_2)\cdots
        \Gamma_{\mathcal{Z}}(z_m),
    \]
    and, for every source sequence,
    $
        \Gamma_{\mathcal{Z}}\!\left(\Pi_{\mathcal{Z}}(y_1^n)\right)
        =
        y_1^n .
    $
    
    This abstraction captures the core parsing mechanism behind many practical
    dictionary and subword tokenizers: a learned vocabulary constrains which strings
    can be emitted, and the parser segments the input into vocabulary elements. It
    covers classical dictionary methods such as LZW~\citep{zouharTokenizationNoiselessChannel2023}, and is a useful
    model for greedy subword tokenizers such as WordPiece~\citep{devlinBERTPretrainingDeep2019}
    and variants of Byte Pair Encoding~\citep{sennrichNeuralMachineTranslation2016}.
    
    
    
    
    \begin{figure}[t]
    \centering
    \begin{minipage}[c]{0.30\textwidth}
    \centering
    \begin{tikzpicture}[
        node distance=0pt,
        treenode/.style={circle, draw, thick, minimum size=0.75cm, inner sep=1pt, font=\small},
        elabel/.style={font=\scriptsize, inner sep=1pt}
    ]
    \node[treenode, fill=gray!15] (root) at (0, 2.2) {$\varepsilon$};
    \node[treenode]               (n0)   at (-1.1, 1.1) {$0$};
    \node[treenode]               (n1)   at ( 1.1, 1.1) {$1$};
    \node[treenode]               (n01)  at (-1.1, 0.0) {$01$};
    \node[treenode]               (n010)  at (-1.1, -1.1) {$010$};
    
    \draw[thick] (root) -- (n0)
        node[elabel, midway, above left=1pt] {$0$};
    \draw[thick] (root) -- (n1)
        node[elabel, midway, above right=1pt] {$1$};
    \draw[thick] (n0) -- (n01)
        node[elabel, midway, right=2pt] {$1$};
    \draw[thick] (n01) -- (n010)
        node[elabel, midway, right=2pt] {$0$};
    
    \end{tikzpicture}
    \\[0.6em]
    {\small (a) Prefix tree $\mathcal{T}$ with vocabulary $\mathcal{Z} = \{0,\, 1,\, 01,\, 010\}$.}
    \end{minipage}
    \hfill
    \begin{minipage}[c]{0.66\textwidth}
    \centering
    \begin{tikzpicture}[
        yscale=1.0, xscale=1.0,
        ysym/.style={font=\small\ttfamily, minimum width=0.5cm, minimum height=0.45cm, draw, thick},
        tok/.style={font=\small, minimum height=0.45cm, draw, thick, fill=blue!12, inner sep=3pt},
        bracket/.style={thick, decorate, decoration={brace, amplitude=3pt, mirror}},
        label/.style={font=\scriptsize}
    ]
    \node[label, anchor=east] at (-0.3, 1.0) {$y_1^{10}:$};
    \foreach \i/\v in {1/0, 2/1, 3/0, 4/1, 5/1, 6/1, 7/0, 8/1, 9/0, 10/0} {
        \node[ysym] at (\i*0.55, 1.0) {\v};
    }
    
    \node[label, anchor=east] at (-0.3, -0.5) {$z_1^{6}:$};
    \node[tok, minimum width=1.65cm] at (1.075, -0.5) {$010$};
    \node[tok, minimum width=0.5cm] at (2.2, -0.5) {$1$};
    \node[tok, minimum width=0.5cm] at (2.75, -0.5) {$1$};
    \node[tok, minimum width=0.5cm]  at (3.30,  -0.5) {$1$};
    \node[tok, minimum width=1.60cm]  at (4.395,  -0.5) {$010$};
    \node[tok, minimum width=0.5cm]  at (5.50,  -0.5) {$0$};
    
    \draw[bracket] (0.30, 0.72) -- (1.90, 0.72);
    \draw[bracket] (1.95, 0.72) -- (2.465, 0.72);
    \draw[bracket] (2.475, 0.72) -- (3.00, 0.72);
    \draw[bracket] (3.05, 0.72) -- (3.55, 0.72);
    \draw[bracket] (3.60, 0.72) -- (5.20, 0.72);
    \draw[bracket] (5.25, 0.72) -- (5.75, 0.72);
    
    \end{tikzpicture}
    \\[0.6em]
    {\small (b) Greedy parsing $\Pi_{\mathcal{Z}}(y_1^{10})$ for $y_1^{10} = 0101110100$ produces $z_1^{6} = (010,\, 1,\, 1,\, 1,\, 010,\, 0)$. Each bracketed source segment maps to a single token below.}
    \end{minipage}
    
    \caption{Greedy tokenization on a small example. Left: a prefix tree with vocabulary $\mathcal{Z}=\{0,\,1, 01, \,010\}$. Right: the parser scans left to right, at each position emitting the deepest matching node.}
    \label{fig:greedy_tokenization}
    \end{figure}
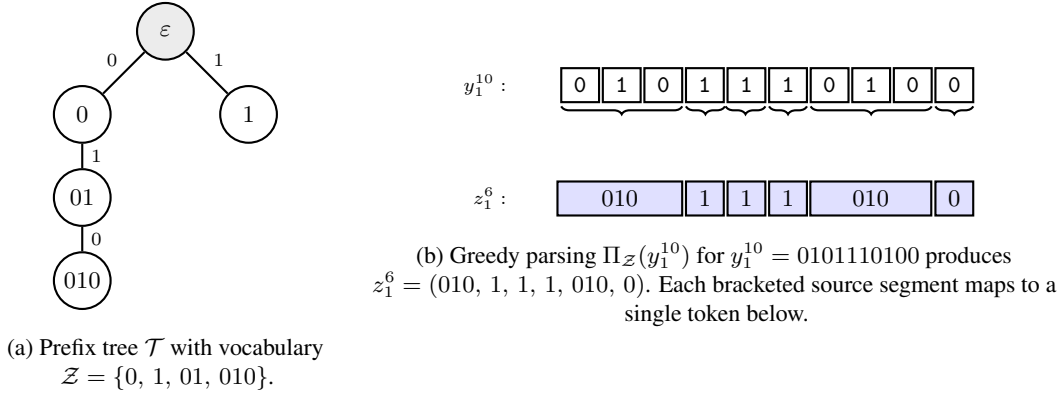
    
    \paragraph{Source span of a token window.}
    We can now make the context-extension intuition precise. A token predictor does
    not directly condition on source symbols; it conditions on the last $w$ tokens.
    The relevant question is how much source history those tokens actually contain.
    For a token string $z_{i-w}^{i-1}$, define its source span
    $S:\mathcal{Z}^*\to\mathbb{N}$ by
    \[
        S(z_{i-w}^{i-1})
        :=
        \left|
            \Gamma_{\mathcal{Z}}(z_{i-w}^{i-1})
        \right|.
    \]
    In particular, $S(z_{i-w}^{i-1})$ is the number of original source symbols
    contained in the token context seen by the predictor.
    
    To compare token-level and source-level prediction, we again measure loss in
    bits per original source symbol. Let $\mathcal{Q}_w^{\mathcal{Z}}$ denote the
    class of $w$-context predictors over the token alphabet $\mathcal{Z}$. Note that stationarity and ergodicity of $\{Y_i\}$ implies that $\{Z_j\}$ is also stationary and ergodic.\footnote{A one-sided greedy parse has an initial-boundary artifact, since the first token is forced to start at the origin. Throughout, $\{Z_j\}$ denotes the stationary version obtained by looking at the parse from a typical emitted-token boundary. Equivalently, the expectations below can be read as long-run empirical averages along a long greedy parse, with boundary effects ignored.} For
    $\tilde q\in\mathcal{Q}_w^{\mathcal{Z}}$, define
    \[
        L_Z^{\mathrm{tok}}(\tilde q)
        :=
        \frac{
            \mathbb{E}\!\left[
                -\log_2
                \tilde q\!\left(
                    Z_0 \mid Z_{-w}^{-1}
                \right)
            \right]
        }{
            \mathbb{E}\!\left[
                |\Gamma_{\mathcal{Z}}(Z_0)|
            \right]
        }.
    \]
    The denominator converts token-level loss into bits per source symbol. The
    optimal tokenized loss with context length $w$ is
    \[
        \mathcal{L}_Z^{\mathrm{tok}}(w)
        :=
        \min_{\tilde q\in\mathcal{Q}_w^{\mathcal{Z}}}
        L_Z^{\mathrm{tok}}(\tilde q).
    \]
    
    
    We now ask when the source span of a token window can be turned into an actual
    prediction guarantee. For a fixed token-window length $w$, define the worst-case
    source span
    \[
        w_s(w)
        :=
        \min\left\{
            S(z_1^w)
            :
            z_1^w \text{ is a valid length-$w$ token window}
        \right\}.
    \]
    Thus every valid $w$-token context contains at least $w_s(w)$ source symbols.
    This means that the token history contains the source history used by a
    $w_s(w)$-context source predictor.
    
    However, this containment alone does not give a token-level predictor. A source
    predictor assigns probabilities one source symbol at a time, while a token
    predictor must assign probabilities to entire tokens. Moreover, under greedy
    parsing, not every source string that has positive probability can be emitted as
    the next token: shorter strings may be suppressed when they are prefixes of
    longer matching vocabulary elements. Thus, one must construct a valid
    distribution over next tokens whose cumulative probability along the greedy
    parse matches the source-level predictor.
    
    The next theorem shows that this can be done exactly. Any source predictor with
    context length $w_s(w)$, with some mild assumptions, can be converted into a $w$-token predictor whose
    cumulative log-loss matches the source-level cumulative log-loss up to an
    additive $O(1)$ term independent of the sequence length.
    
    \begin{mainthmbox}
    \begin{theorem}[Predictor transfer under greedy tokenization]
    \label{thm:predictor_transfer}
    Fix a tokenizer $\Pi_{\mathcal{Z}}$ and a token-window length $w\ge 1$, and let
    \[
        w_s := w_s(w)
        =
        \min\left\{
            S(z_1^w)
            :
            z_1^w \text{ is a valid length-$w$ token window}
        \right\}.
    \]
    Let $q\in\mathcal{Q}_{w_s}^{\mathcal{Y}}$ be strictly positive, meaning that
    there exists $\lambda_q>0$ such that
    \[
        q(y\mid c)\ge \lambda_q,
        \qquad
        \forall y\in\mathcal Y,\; c\in\mathcal Y^{w_s}.
    \]
    Then there exists a token-level predictor
    $\tilde q\in\mathcal{Q}_w^{\mathcal{Z}}$ such that, for every source sequence
    $y_1^n$,
    \[
        \ell\!\left(\tilde q;\Pi_{\mathcal{Z}}(y_1^n)\right)
        =
        \ell(q;y_1^n)
        +
        O(1).
    \]
    \end{theorem}
    \end{mainthmbox}
    The proof is constructive, deferred to Appendix~\ref{proof:predictor-transfer}.
    
    The theorem shows that source span is not only an intuitive measure of context length; it is enough to transfer prediction rules. Whenever every $w$-token window spans at least $w_s$ source symbols, a $w$-token predictor can reproduce the behavior of any $w_s$-context source predictor, with no asymptotic loss.

    This immediately gives a guarantee on the optimal token-level loss.
    
    \begin{mainthmbox}
    \begin{corollary}
    \label{cor:tok_no_loss}
    Let $\{Y_i\}$ be an ergodic stationary $k^{\text{th}}$-order Markov process on
    $\mathcal{Y}$, and let $\{Z_j\}=\Pi_{\mathcal{Z}}(\{Y_i\})$ be its tokenized
    process. For any token-window length $w\ge 1$, let
    \[
        w_s(w)
        :=
        \min\left\{
            S(z_1^w)
            :
            z_1^w \text{ is a valid length-$w$ token window}
        \right\}.
    \]
    Then
    \[
        \mathcal{L}_Z^{\mathrm{tok}}(w)
        \le
        \mathcal{L}_Y(w_s(w)).
    \]
    \end{corollary}
    \end{mainthmbox} --- the proof is given in Appendix~\ref{proof:tok_no_loss}

    For a $k$th-order Markov source, $\mathcal{L}_Y(w_s)$ reaches the entropy rate once $w_s\ge k$. Thus a token predictor with window $w<k$ can still achieve the Markov optimum, provided that its $w$ tokens always span at least $k$ source symbols. This explains the empirical observation: Even though greedy parsing introduces boundary-dependent constraints on which tokens can appear next, these constraints do not prevent transfer --- a token-level predictor
can still match the loss of a source-level predictor with the corresponding
extended source context.
    
\paragraph{From worst-case to typical span.}
The deterministic condition above captures the mechanism, but it is too
stringent for practical tokenizers.  Real vocabularies contain short tokens,
often including single-symbol tokens.  Hence the worst-case span of a
$w$-token window may be small even when most windows cover much more source
history.  To capture this behavior, we relax the span requirement: instead of
requiring every $w$-token window to cover $w_s$ source symbols, we require this with high probability.

\begin{definitionbox}
\begin{definition}[Typical source span]
\label{def:context_extending}
Let $\{Y_i\}$ be a stationary source on $\mathcal{Y}$, and let
$\Pi_{\mathcal{Z}}$ be a tokenizer over $\mathcal{Y}^*$.  Fix a token-window
length $w$ and a target source span $w_s$.  We say that
$\Pi_{\mathcal{Z}}$ is $(\epsilon,w,w_s)$-typical with respect to $\{Y_i\}$ if
\[
    \Pr\!\left[
        S(Z_1^w) < w_s
    \right]
    \le
    \epsilon,
\]
where $Z_1^w$ is a window of $w$ consecutive tokens drawn from the tokenized
process.
\end{definition}
\end{definitionbox}
    
    Here $\epsilon$ is the probability that a $w$-token context fails to contain
    $w_s$ source symbols. When $\epsilon=0$, The deterministic setting of
    Corollary~\ref{cor:tok_no_loss} is recovered. When $\epsilon>0$, the transfer may fail on a
    small fraction of windows, and the loss guarantee must pay for those failures.
    To express this penalty per source symbol, we define the expected token length
    $
        \alpha_{\mathcal{Z}}
        :=
        \mathbb{E}\!\left[
            |\Gamma_{\mathcal{Z}}(Z_0)|
        \right]
    $
    and the
    tokenizer compression rate
    \[
        R_{\mathcal{Z}}
        :=
        \frac{\log_2 |\mathcal{Z}|}
        {\alpha_{\mathcal{Z}}\log_2|\mathcal{Y}|}.
    \]
    Thus $R_{\mathcal{Z}}$ is the uniform-code rate of the tokenizer, measured in units of source-symbol bits.
    
    \begin{mainthmbox}
    \begin{theorem}[Tokenization gain]\label{thm:tokenization_gain}
    Let $\{Y_i\}$ be a stationary $k^{\text{th}}$-order Markov process on
    $\mathcal{Y}$, and let $\Pi_{\mathcal{Z}}$ be a tokenizer that is
    $(\epsilon,w,w_s)$-typical with respect to $\{Y_i\}$. Then
    \[
        \mathcal{L}_Z^{\mathrm{tok}}(w)
        \le
        \mathcal{L}_Y(w_s)
        +
        \epsilon\,R_{\mathcal{Z}}\log_2|\mathcal{Y}|.
    \]
    \end{theorem}
    \end{mainthmbox}
    \paragraph{Interpretation.}
Theorem~\ref{thm:tokenization_gain} converts a typical-span statement into a
prediction guarantee.  For a fixed token window $w$, one may choose the
largest source span $w_s$ for which the slack $\epsilon(w,w_s)R_{\mathcal Z}\log_2|\mathcal Y|$
remains small.  The theorem then shows that the best $w$-context token
predictor matches a $w_s$-context source predictor up to this slack.  In this
sense, $w_s$ is an effective source context supplied by the tokenizer at
window length $w$.  The proof is deferred to
Appendix~\ref{proof:tokenization_gain}.
    
    
    
    
    \begin{figure}[h]
    \centering

    \begin{minipage}[c]{0.44\linewidth}
    \raggedright
   \textbf{Measuring effective context in real tokenizers.}
The theorem suggests a direct diagnostic.  Fix a token-window length $w$,
vary the target source span $w_s$, and estimate
  $  \epsilon(w,w_s)
    =
    \Pr\!\left[S(Z_1^w)<w_s\right].$
Equivalently, plot the slack $
    \epsilon(w,w_s)\,R_{\mathcal{Z}}\log_2|\mathcal{Y}|.$
Figure~\ref{fig:cdf_scaled} shows this quantity on WikiText-103 for trained
and pretrained tokenizers.  Across tokenizers, the slack is close to zero for source spans up to roughly $3$--$5$ times the token-window length.
    
    \end{minipage}
    \hfill
    \begin{minipage}[c]{0.54\linewidth}
        \centering
        \includegraphics[width=0.9\linewidth]{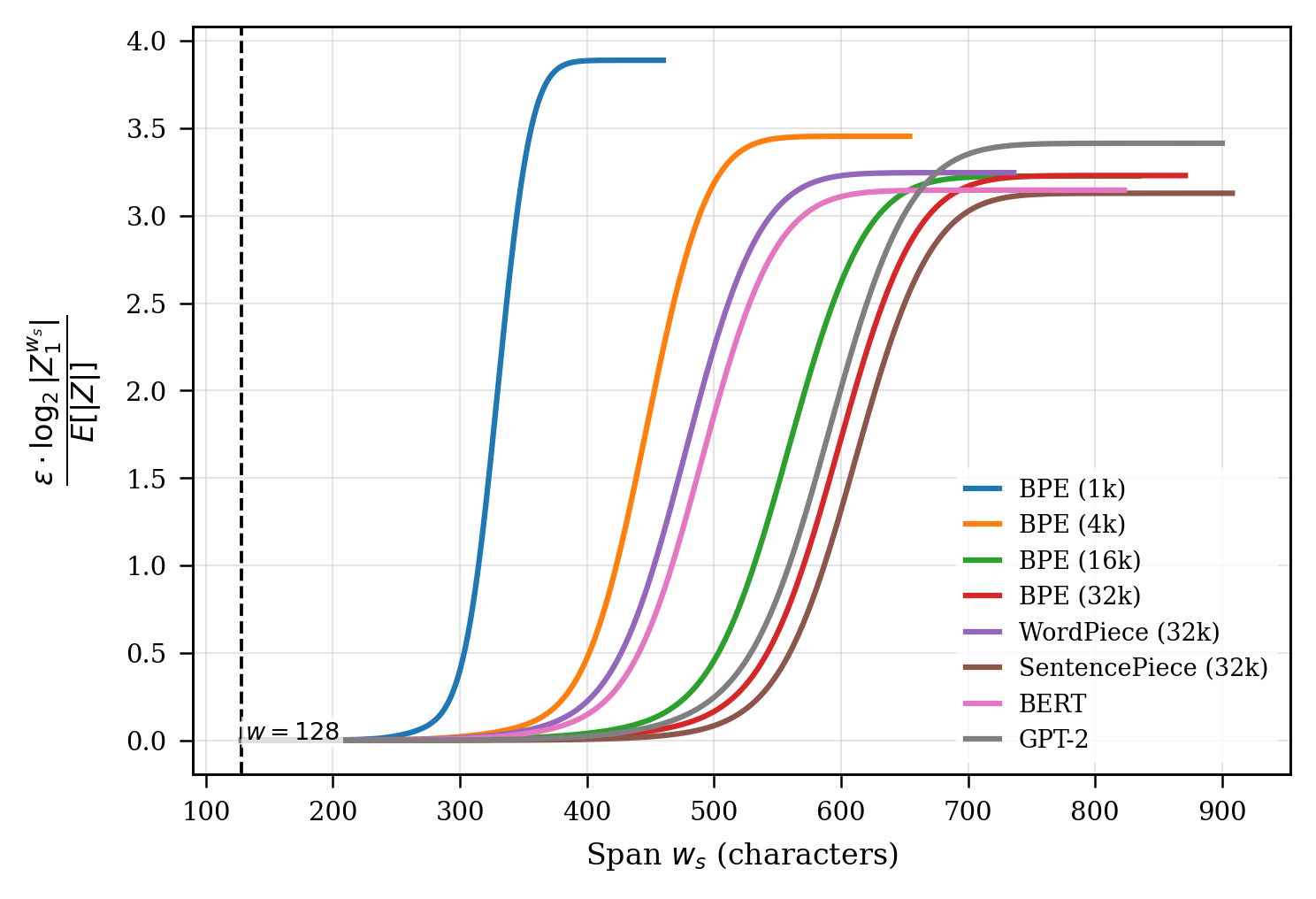}
        \captionof{figure}{
    Slack term $\epsilon(w,w_s)R_{\mathcal Z}\log_2|\mathcal Y|$ versus target
    source span $w_s$ for trained and pretrained tokenizers on WikiText-103, with
    token-window length $w=128$. Slack is near zero up to
    $w_s \approx 3$--$5\times w$.
    }
    \label{fig:cdf_scaled}
    \end{minipage}
    
    \end{figure}
    \paragraph{Existence under a vocabulary budget.}
    The typical-span condition is not only a diagnostic. In
    Appendix~\ref{app:typical-heavy-hitting-tokenizers}, we show under heavy-hitting assumptions for learned greedy tokenizers
    as in \citep{rajaramanAnalysisTokenizationTransformers2024}, a budget-$d$ greedy
    tokenizer can satisfy the typical-span condition with a target source span that
    grows like $w\log d$. Thus, increasing the vocabulary budget can provably make a
    fixed $w$-token window simulate a much longer source-context window, up to
    the slack term in Thm~\ref{thm:tokenization_gain}.
    
    \section{Open Questions and Limitations}
    \label{sec:discussion}
    
    Our results isolate an achievable-loss effect, but not the full optimization or
    expressivity picture for trained Transformers. The predictor-transfer arguments
    construct finite-context predictors, but they do not show that a 
    Transformer architecture can realize them, or that training will find them. In
    Appendix~\ref{app:matched_plateau}, tokenization appears to
    improve convergence speed; explaining this remains open.
    
    Our theory also gives a criterion for tokenizer quality, not a complete
    tokenizer-design method. It suggests that average token length is insufficient:
    the lower tail of the source-span distribution matters. This motivates training
    tokenizers to optimize effective context while balancing vocabulary size,
    output-layer cost, and rare-token effects.

    \section{Conclusion}
    We studied how lossless representations affect finite-context prediction.
    Fragmentation can hurt by introducing phase ambiguity and missing source
    history, even though the original sequence is recoverable. This gives a
    theoretical account of the finite-context gap observed in byte- and
    character-level models such as ByT5 and CANINE relative to subword-tokenized
    models. Tokenization---BPE, WordPiece, and related methods---can have the
    opposite effect: a short token window can behave like a longer source window
    when it reliably spans enough source symbols. Together, the results show that
    the relevant quantity for Transformers is not only raw context length, but the
    effective source context exposed by the representation.

    \bibliographystyle{plainnat}
    \bibliography{references}
    \newpage
    \appendix
    
    \section{Proof of Theorems}
    \subsection{Proof of Theorem~\ref{thm:fragmentation}}
    \label{proof:fragmentation}
    
    \begin{proof}
    Recall that the process \(\{X_j\}_{j\in\mathbb Z}\) is cyclostationary with
    period \(M\). Indeed, shifting the fragmented process by \(M\) positions
    corresponds exactly to shifting the original source process by one symbol while
    preserving the phase within each fragmented block. Since \(\{Y_i\}\) is
    stationary, for any indices \(j_1,\ldots,j_r\),
    \[
        (X_{j_1},\ldots,X_{j_r})
        \overset{d}{=}
        (X_{j_1+M},\ldots,X_{j_r+M}).
    \]
    
    We now randomize the phase. Let \(\Theta\) be uniform on
    \(\{1,\ldots,M\}\), independent of \(\{X_j\}\), and define
    \[
        X'_j := X_{j+\Theta},
        \qquad j\in\mathbb Z.
    \]
    By the law of total probability over \(\Theta\),
    \begin{align*}
    &\Pr\!\left(
        X'_{j_1}=x_1,\ldots,X'_{j_r}=x_r
    \right) \\
    &\qquad =
    \sum_{\theta=1}^{M}
    \Pr\!\left(
        X_{j_1+\theta}=x_1,\ldots,X_{j_r+\theta}=x_r
        \,\middle|\, \Theta=\theta
    \right)
    \Pr(\Theta=\theta).
    \end{align*}
    Since \(\Theta\) is independent of \(\{X_j\}\) and uniform on
    \(\{1,\ldots,M\}\), this gives
    \begin{align*}
    &\Pr\!\left(
        X'_{j_1}=x_1,\ldots,X'_{j_r}=x_r
    \right) \\
    &\qquad =
    \frac{1}{M}
    \sum_{\theta=1}^{M}
    \Pr\!\left(
        X_{j_1+\theta}=x_1,\ldots,X_{j_r+\theta}=x_r
    \right).
    \end{align*}
    Thus the joint distribution of \(\{X'_j\}\) is obtained by
    averaging over the \(M\) phases of the cyclostationary process
    \(\{X_j\}\). Since shifting all indices of \(X'\) only permutes these \(M\)
    phase classes, \(\{X'_j\}\) is stationary.
    
    For any \(q\in\mathcal{Q}_{Mw}^{\mathcal X}\), by definition,
    \[
        L_X^{\mathrm{frag}}(q)
        =
        \sum_{\theta=1}^{M}
        \mathbb{E}\!\left[
            -\log_2
            q\!\left(
                X_{\theta}
                \mid
                X_{\theta-Mw}^{\theta-1}
            \right)
        \right].
    \]
    Using the phase-randomized process, by the law of total expectation,
    \begin{align*}
    &\mathbb{E}\!\left[
        -\log_2
        q\!\left(
            X'_0
            \mid
            (X')_{-Mw}^{-1}
        \right)
    \right] \\
    &\qquad =
    \mathbb{E}\!\left[
        -\log_2
        q\!\left(
            X_{\Theta}
            \mid
            X_{\Theta-Mw}^{\Theta-1}
        \right)
    \right] \\
    &\qquad =
    \frac{1}{M}
    \sum_{\theta=1}^{M}
    \mathbb{E}\!\left[
        -\log_2
        q\!\left(
            X_{\theta}
            \mid
            X_{\theta-Mw}^{\theta-1}
        \right)
    \right].
    \end{align*}
    Therefore,
    \[
        L_X^{\mathrm{frag}}(q)
        =
        M \cdot
        \mathbb{E}\!\left[
            -\log_2
            q\!\left(
                X'_0
                \mid
                (X')_{-Mw}^{-1}
            \right)
        \right].
    \]
    Since \(\{X'_j\}\) is stationary, the optimal \(Mw\)-context predictor for
    \(X'\) achieves the conditional entropy. Hence
    \begin{equation}
        \mathcal{L}_X^{\mathrm{frag}}(Mw)
        =
        M H\!\left(
            X'_0
            \mid
            (X')_{-Mw}^{-1}
        \right)
        =
        M H\!\left(
            X_{\Theta}
            \mid
            X_{\Theta-Mw}^{\Theta-1}
        \right).
        \label{eq:frag-loss-phase-randomized}
    \end{equation}
    
    We now express the source-level loss in terms of the fragmented variables. Since
    \(\varphi\) is injective, the block \(X_1^M\) contains exactly the same
    information as \(Y_1\), and the block \(X_{1-Mw}^{0}\) contains exactly the
    same information as \(Y_{1-w}^{0}\). Therefore,
    \[
        \mathcal{L}_Y(w)
        =
        H(Y_1\mid Y_{1-w}^{0})
        =
        H(X_1^M\mid X_{1-Mw}^{0}).
    \]
    Applying the chain rule,
    \begin{equation}
        \mathcal{L}_Y(w)
        =
        \sum_{\theta=1}^{M}
        H\!\left(
            X_{\theta}
            \mid
            X_{1-Mw}^{\theta-1}
        \right).
        \label{eq:Y-loss-chain-frag}
    \end{equation}
    
    Next, decompose the phase-randomized loss. Since \(\Theta\) is uniform,
    \[
        H\!\left(
            X_{\Theta}
            \mid
            X_{\Theta-Mw}^{\Theta-1},
            \Theta
        \right)
        =
        \frac{1}{M}
        \sum_{\theta=1}^{M}
        H\!\left(
            X_{\theta}
            \mid
            X_{\theta-Mw}^{\theta-1}
        \right).
    \]
    Therefore,
    \begin{align}
    & M H\!\left(
            X_{\Theta}
            \mid
            X_{\Theta-Mw}^{\Theta-1}
        \right)
        -
        \sum_{\theta=1}^{M}
        H\!\left(
            X_{\theta}
            \mid
            X_{\theta-Mw}^{\theta-1}
        \right)
        \nonumber \\
    &\qquad =
        M I\!\left(
            \Theta;
            X_{\Theta}
            \mid
            X_{\Theta-Mw}^{\Theta-1}
        \right).
        \label{eq:phase-ambiguity-step}
    \end{align}
    
    Using \eqref{eq:frag-loss-phase-randomized} and
    \eqref{eq:Y-loss-chain-frag}, we add and subtract
    \(\sum_{\theta=1}^{M}
    H(X_{\theta}\mid X_{\theta-Mw}^{\theta-1})\) to obtain
    \begin{align*}
    &\mathcal{L}_X^{\mathrm{frag}}(Mw) - \mathcal{L}_Y(w) \\
    &=
    \Bigg[
        M H\!\left(
            X_{\Theta}
            \mid
            X_{\Theta-Mw}^{\Theta-1}
        \right)
        -
        \sum_{\theta=1}^{M}
        H\!\left(
            X_{\theta}
            \mid
            X_{\theta-Mw}^{\theta-1}
        \right)
    \Bigg] \\
    &\quad+
    \sum_{\theta=1}^{M}
    \Bigg[
        H\!\left(
            X_{\theta}
            \mid
            X_{\theta-Mw}^{\theta-1}
        \right)
        -
        H\!\left(
            X_{\theta}
            \mid
            X_{1-Mw}^{\theta-1}
        \right)
    \Bigg].
    \end{align*}
    The first bracket is the phase-ambiguity term by
    \eqref{eq:phase-ambiguity-step}. For the second term, observe that when
    \(\theta=1\),
    \[
        X_{\theta-Mw}^{\theta-1}
        =
        X_{1-Mw}^{0}
        =
        X_{1-Mw}^{\theta-1},
    \]
    so the corresponding difference is zero. For \(\theta\ge 2\),
    \[
        X_{1-Mw}^{\theta-1}
        =
        \left(
            X_{1-Mw}^{\theta-Mw-1},
            X_{\theta-Mw}^{\theta-1}
        \right).
    \]
    Hence
    \begin{align*}
    & H\!\left(
            X_{\theta}
            \mid
            X_{\theta-Mw}^{\theta-1}
        \right)
        -
        H\!\left(
            X_{\theta}
            \mid
            X_{1-Mw}^{\theta-1}
        \right) \\
    &\qquad =
        I\!\left(
            X_{\theta};
            X_{1-Mw}^{\theta-Mw-1}
            \,\middle|\,
            X_{\theta-Mw}^{\theta-1}
        \right).
    \end{align*}
    Combining the two parts gives
    \[
    \mathcal{L}_X^{\mathrm{frag}}(Mw) - \mathcal{L}_Y(w)
    =
    \sum_{\theta = 2}^{M}
    I\!\left(
        X_{\theta};
        X_{1-Mw}^{\theta-Mw - 1}
        \,\middle|\,
        X_{\theta-Mw}^{\theta-1}
    \right)
    +
    M I\!\left(
        \Theta;
        X_{\Theta}
        \,\middle|\,
        X_{\Theta-Mw}^{\Theta-1}
    \right),
    \]
    which is the desired decomposition.
    
    It remains to justify the final claim. The context-deficit term measures the
    information about \(X_\theta\) contained in the missing prefix
    \(X_{1-Mw}^{\theta-Mw-1}\) beyond the fragmented context
    \(X_{\theta-Mw}^{\theta-1}\). When \(w>k\), the fragmented context
    \(X_{\theta-Mw}^{\theta-1}\) already contains the last \(k\) complete source
    symbols before \(Y_1\), namely \(Y_{1-k}^{0}\), as well as the previous
    fragments \(X_1^{\theta-1}\) of the current symbol. The missing prefix belongs
    only to source symbols older than this Markov context. Since \(\{Y_i\}\) is
    \(k\)-th order Markov and \(X_\theta\) is a deterministic function of \(Y_1\),
    the missing prefix carries no additional information about \(X_\theta\) once
    \(X_{\theta-Mw}^{\theta-1}\) is given. Therefore,
    \[
        I\!\left(
            X_{\theta};
            X_{1-Mw}^{\theta-Mw - 1}
            \,\middle|\,
            X_{\theta-Mw}^{\theta-1}
        \right)
        =
        0,
        \qquad \theta=2,\ldots,M,
    \]
    whenever \(w>k\). Thus the context-deficit term vanishes for \(w>k\).
    \end{proof}

    \newpage
    \subsection{Proof of Theorem~\ref{thm:predictor_transfer}}
    \label{proof:predictor-transfer}
    
    \begin{proof}
    Throughout this proof, we use the prefix-closedness of the vocabulary:
    whenever \(z\in\mathcal Z\), every nonempty prefix of \(z\) also belongs to
    \(\mathcal Z\). Therefore, if a token \(z\) can be extended inside the
    vocabulary, it can be extended by one source symbol at a time.
    
    For \(z\in\mathcal Z\), define the set of one-symbol vocabulary extensions of
    \(z\) by
    \[
        \operatorname{Ext}_{\mathcal Z}(z)
        :=
        \left\{
            a\in\mathcal Y :
            za\in\mathcal Z
        \right\}.
    \]
    Thus \(\operatorname{Ext}_{\mathcal Z}(z)\) is the set of source symbols by
    which the current token \(z\) can be extended while remaining inside the
    vocabulary.

    For a string \(s=s_1\cdots s_m\in\mathcal Y^*\) with \(m\ge w_s\), define
    \[
        \operatorname{suf}_{w_s}(s)
        :=
        s_{m-w_s+1}^m .
    \]
    That is, \(\operatorname{suf}_{w_s}(s)\) is the suffix of \(s\) of length
    \(w_s\).
    
    Given a source-level predictor \(q\in\mathcal Q_{w_s}^{\mathcal Y}\), define
    its sequential extension \(q^{\mathrm{seq}}\) by
    \[
        q^{\mathrm{seq}}(u_1^r\mid h)
        :=
        \prod_{j=1}^{r}
        q\!\left(
            u_j
            \,\middle|\,
            \operatorname{suf}_{w_s}(h u_1^{j-1})
        \right),
    \]
    for every \(u_1^r\in\mathcal Y^*\) and every history \(h\in\mathcal Y^*\) with
    \(|h|\ge w_s\).
    
    We now construct the token-level predictor induced by \(q\). For readability,
    write
    \[
        \gamma(z) := \Gamma_{\mathcal Z}(z),
        \qquad z\in\mathcal Z .
    \]
    For a token context \(z_{i-w}^{i-1}\), define its expanded source history by
    \[
        H_i
        :=
        \Gamma_{\mathcal Z}(z_{i-w}^{i-1}).
    \]
    By the definition of \(w_s=w_s(w)\), every valid \(w\)-token context spans at
    least \(w_s\) source symbols. Hence \(q^{\mathrm{seq}}(\cdot\mid H_i)\) is
    well-defined.
    
    The token predictor must assign probabilities to tokens, not merely to source
    symbols. Under greedy parsing, a candidate token \(z_i\) is emitted exactly when
    the source continuation first matches \(\gamma(z_i)\), and then stops there
    rather than extending \(z_i\) by another source symbol in
    \(\operatorname{Ext}_{\mathcal Z}(z_i)\). In addition, since the previous token
    \(z_{i-1}\) has already been emitted, we must condition on the fact that it was
    not extended; equivalently, the next source symbol is not in
    \(\operatorname{Ext}_{\mathcal Z}(z_{i-1})\).
    
    Define \(\tilde q\in\mathcal Q_w^{\mathcal Z}\) by
    \[
    \tilde q(z_i\mid z_{i-w}^{i-1})
    :=
    \begin{cases}
    \displaystyle
    \frac{
    q^{\mathrm{seq}}\!\left(\gamma(z_i)\mid H_i\right)
    \left(
    1
    -
    \sum_{a\in\operatorname{Ext}_{\mathcal Z}(z_i)}
    q\!\left(
        a
        \,\middle|\,
        \operatorname{suf}_{w_s}(H_i\gamma(z_i))
    \right)
    \right)
    }{
    1
    -
    \sum_{a\in\operatorname{Ext}_{\mathcal Z}(z_{i-1})}
    q\!\left(
        a
        \,\middle|\,
        \operatorname{suf}_{w_s}(H_i)
    \right)
    },
    &
    \text{if } \gamma(z_i)_1\notin\operatorname{Ext}_{\mathcal Z}(z_{i-1}),
    \\[3ex]
    0,
    &
    \text{otherwise.}
    \end{cases}
    \]
    Here \(\gamma(z_i)_1\) denotes the first source symbol of the token \(z_i\).
    
    The numerator is the probability, under the source predictor \(q\), that the
    next emitted token is exactly \(z_i\): the continuation matches
    \(\gamma(z_i)\), and the following source symbol does not extend \(z_i\) to a
    longer vocabulary token. The denominator conditions on the event that the
    previous token \(z_{i-1}\) really ended, rather than being extended. The
    condition
    \[
        \gamma(z_i)_1\notin\operatorname{Ext}_{\mathcal Z}(z_{i-1})
    \]
    rules out candidates that could not legally follow \(z_{i-1}\) under greedy
    parsing.
    
    Because the sets \(\operatorname{Ext}_{\mathcal Z}(z)\) contain one-symbol
    extensions, the extension events in the sums above are disjoint. Moreover, the
    legal next-token events partition the event that \(z_{i-1}\) is not extended.
    Therefore the displayed rule defines a valid probability distribution over
    next tokens.
    
    We also record an important property of the construction. Apart from the
    identity of the immediately preceding token \(z_{i-1}\), which is needed to
    enforce greedy parsing, the predictor \(\tilde q\) depends on the expanded
    history \(H_i\) only through
    \[
        \operatorname{suf}_{w_s}(H_i).
    \]
    Indeed, every occurrence of \(q\) in the definition above is evaluated using
    only the suffix of length \(w_s\) of the relevant source history. This is the
    sense in which the transferred predictor uses only the \(w_s\)-symbol source
    suffix of the token context.
    
    We now verify that the cumulative log-loss matches that of \(q\) up to boundary
    terms. Fix a source sequence \(y_1^n\), and let
    \[
        z_1^m := \Pi_{\mathcal Z}(y_1^n)
    \]
    be its greedy tokenization. Thus
    \[
        \gamma(z_1)\gamma(z_2)\cdots\gamma(z_m)
        =
        y_1^n .
    \]
    
    Note that the normalizing denominators in the construction below are
    strictly positive on every valid greedy-parsing context. Indeed, if a token
    $z\in\mathcal Z$ is emitted at an interior token boundary, then it cannot be
    extendable by every source symbol. Otherwise, if
    $\operatorname{Ext}_{\mathcal Z}(z)=\mathcal Y$, the next source symbol would
    always extend $z$ to a longer vocabulary element, and greedy parsing would not
    emit $z$ at that boundary. Hence, for every emitted interior token $z$, there
    exists at least one $a\in\mathcal Y\setminus\operatorname{Ext}_{\mathcal Z}(z)$.
    Since $q$ is strictly positive, we have
    \[
        1-
        \sum_{a\in\operatorname{Ext}_{\mathcal Z}(z)}
        q(a\mid c)
        =
        \sum_{a\notin\operatorname{Ext}_{\mathcal Z}(z)}
        q(a\mid c)
        \ge \lambda_q
    \]
    for every context $c\in\mathcal Y^{w_s}$. Therefore the denominators appearing
    below are well-defined and bounded away from zero along valid greedy parses.
    For invalid token contexts, we define the predictor arbitrarily, for instance
    uniformly over $\mathcal Z$.

    We ignore the first \(w\) tokens and the corresponding initial source symbols.
    Since the vocabulary is finite, this initial part has bounded source length
    depending only on \(w\) and the maximum token length. Its contribution to the
    cumulative log-loss is therefore an \(O(1)\) boundary term.
    
    For \(i>w\), define
    \[
        R_i
        :=
        1
        -
        \sum_{a\in\operatorname{Ext}_{\mathcal Z}(z_i)}
        q\!\left(
            a
            \,\middle|\,
            \operatorname{suf}_{w_s}
            \!\left(
                \Gamma_{\mathcal Z}(z_{i-w+1}^{i})
            \right)
        \right).
    \]
    The quantity \(R_i\) is the \(q\)-probability that, after token \(z_i\) has been
    emitted, greedy parsing stops at \(z_i\) rather than extending it by one more
    source symbol.
    
    For \(i>w\), we have
    \[
        H_i\gamma(z_i)
        =
        \Gamma_{\mathcal Z}(z_{i-w}^{i}),
        \qquad
        \Gamma_{\mathcal Z}(z_{i-w+1}^{i})
        \text{ is a suffix of }
        \Gamma_{\mathcal Z}(z_{i-w}^{i}).
    \]
    Since every \(w\)-token window spans at least \(w_s\) source symbols, the two
    histories
    \[
        \Gamma_{\mathcal Z}(z_{i-w}^{i})
        \qquad\text{and}\qquad
        \Gamma_{\mathcal Z}(z_{i-w+1}^{i})
    \]
    have the same length-\(w_s\) source suffix. Therefore, the no-extension factor
    in the numerator of \(\tilde q(z_i\mid z_{i-w}^{i-1})\) is exactly \(R_i\),
    while the normalizing factor in the denominator is exactly \(R_{i-1}\). Hence,
    along the valid greedy parse,
    \[
        \tilde q(z_i\mid z_{i-w}^{i-1})
        =
        q^{\mathrm{seq}}\!\left(\gamma(z_i)\mid H_i\right)
        \frac{R_i}{R_{i-1}} .
    \]
    
    Multiplying over \(i=w+1,\ldots,m\), we obtain
    \begin{align*}
    \prod_{i=w+1}^{m}
        \tilde q(z_i\mid z_{i-w}^{i-1})
    &=
    \prod_{i=w+1}^{m}
        q^{\mathrm{seq}}\!\left(\gamma(z_i)\mid H_i\right)
        \prod_{i=w+1}^{m}
        \frac{R_i}{R_{i-1}} \\
    &=
    \prod_{i=w+1}^{m}
        q^{\mathrm{seq}}\!\left(\gamma(z_i)\mid H_i\right)
        \frac{R_m}{R_w}.
    \end{align*}

    Because $q$ is strictly positive, each interior stop probability satisfies
    $\lambda_q\le R_i\le 1$. Hence the telescoping factor $R_m/R_w$ contributes at
    most a constant
    \[
        \left|-\log_2(R_m/R_w)\right|
        \le
        2\log_2(1/\lambda_q),
    \]
    up to the final finite-sequence boundary. Since the vocabulary is finite, the
    initial and terminal boundary regions contain at most $O(1)$
    source symbols, and their contribution is
    $O(1)$.

    It remains to identify the first product. Since
    \[
        \gamma(z_1)\gamma(z_2)\cdots\gamma(z_m)=y_1^n,
    \]
    and since each \(H_i\) contains the same \(w_s\)-symbol source suffix that the
    source-level predictor \(q\) would use immediately before predicting the
    symbols of \(\gamma(z_i)\), we have
    \[
        \prod_{i=w+1}^{m}
        q^{\mathrm{seq}}\!\left(\gamma(z_i)\mid H_i\right)
        =
        \prod_{t\in I_n}
        q\!\left(
            y_t
            \,\middle|\,
            y_{t-w_s}^{t-1}
        \right),
    \]
    where \(I_n\) is the set of source positions remaining after removing the
    initial boundary symbols. Hence this product is exactly the source-level
    probability assigned by \(q\) to \(y_1^n\), up to boundary terms.
    
    Taking negative logarithms gives
    \[
        \ell\!\left(\tilde q;\Pi_{\mathcal Z}(y_1^n)\right)
        =
        \ell(q;y_1^n)
        +
        O(1).
    \]
    This proves the predictor-transfer statement.
    \end{proof}

        
    \subsection{Proof of Corollary~\ref{cor:tok_no_loss}}
    \label{proof:tok_no_loss}
    
    \begin{proof}
    Let $P_{w_s}(\cdot\mid c)$ denote the optimal source-level predictor with
    context length $w_s$, so that
    \[
        \mathcal L_Y(w_s)
        =
        \mathbb E\left[
            -\log_2 P_{w_s}(Y_0\mid Y_{-w_s}^{-1})
        \right].
    \]
    For $\eta\in(0,1)$, define the smoothed predictor
    \[
        q_\eta(y\mid c)
        :=
        (1-\eta)P_{w_s}(y\mid c)
        +
        \eta\frac{1}{|\mathcal Y|}.
    \]
    Then $q_\eta$ is strictly positive, since
    \[
        q_\eta(y\mid c)\ge \frac{\eta}{|\mathcal Y|}.
    \]
    Therefore Theorem~\ref{thm:predictor_transfer} applies to $q_\eta$. Hence there
    exists $\tilde q_\eta\in\mathcal Q_w^{\mathcal Z}$ such that the asymptotic
    tokenized loss of $\tilde q_\eta$ equals the source-level loss of $q_\eta$.
    Consequently,
    \[
        \mathcal L_Z^{\mathrm{tok}}(w)
        \le
        L_Y(q_\eta).
    \]
    Finally, by dominated convergence on the finite alphabet,
    \[
        L_Y(q_\eta)\to \mathcal L_Y(w_s)
        \qquad
        \text{as }\eta\downarrow 0.
    \]
    Thus
    \[
        \mathcal L_Z^{\mathrm{tok}}(w)
        \le
        \mathcal L_Y(w_s).
    \]
    \end{proof}
    
    \subsection{Proof of Theorem~\ref{thm:tokenization_gain}}
    \label{proof:tokenization_gain}
    
    \begin{proof}
    Let $P_{w_s}(\cdot\mid c)$ denote the optimal source-level predictor with
    context length $w_s$, so that
    \[
        \mathcal L_Y(w_s)
        =
        \mathbb E\!\left[
            -\log_2 P_{w_s}(Y_0\mid Y_{-w_s}^{-1})
        \right].
    \]
    For $\eta\in(0,1)$, define the smoothed source predictor
    \[
        q_\eta(y\mid c)
        :=
        (1-\eta)P_{w_s}(y\mid c)
        +
        \eta\frac{1}{|\mathcal Y|},
        \qquad
        y\in\mathcal Y,\; c\in\mathcal Y^{w_s}.
    \]
    Then $q_\eta$ is strictly positive, since
    \[
        q_\eta(y\mid c)\ge \frac{\eta}{|\mathcal Y|}
        \qquad
        \forall y\in\mathcal Y,\; c\in\mathcal Y^{w_s}.
    \]
    Therefore, by Theorem~\ref{thm:predictor_transfer}, applied to $q_\eta$, there
    exists a transferred token-level predictor
    \[
        \tilde q_{\mathrm{tr},\eta}
        \in
        \mathcal Q_{w_s}^{\mathcal Z}
    \]
    whose asymptotic per-source-symbol loss equals the source-level loss of
    $q_\eta$:
    \[
        L_Z^{\mathrm{tok}}(\tilde q_{\mathrm{tr},\eta})
        =
        L_Y(q_\eta).
    \]
    Note that here we intentionally chose the window of the transferred  token-level predictor to be $w_s$ so that it surely spans $w_s$ source symbols to be able to apply Theorem\ref{thm:predictor_transfer}.
    
    We use the suffix-dependence property of the construction in
    Theorem~\ref{thm:predictor_transfer}: the transferred predictor only needs the
    last $w_s$ source symbols contained in the expanded token history. Hence, if a
    $w$-token context $z_{i-w}^{i-1}$ satisfies
    \[
        S(z_{i-w}^{i-1})\ge w_s,
    \]
    then this $w$-token context contains enough source history to evaluate the same
    prediction rule as $\tilde q_{\mathrm{tr},\eta}$.
    
    Define a $w$-context token predictor
    $\tilde q_{\mathrm{typ},\eta}\in\mathcal Q_w^{\mathcal Z}$ by
    \[
    \tilde q_{\mathrm{typ},\eta}(z_i\mid z_{i-w}^{i-1})
    :=
    \begin{cases}
    \tilde q_{\mathrm{tr},\eta}(z_i\mid z_{i-w}^{i-1}),
    &
    \text{if } S(z_{i-w}^{i-1})\ge w_s, \\[1ex]
    \dfrac{1}{|\mathcal Z|},
    &
    \text{if } S(z_{i-w}^{i-1})<w_s.
    \end{cases}
    \]
    In the first case, the notation
    $\tilde q_{\mathrm{tr},\eta}(z_i\mid z_{i-w}^{i-1})$ means that we evaluate the
    transferred predictor using the length-$w_s$ source suffix contained in the
    expanded history $\Gamma_{\mathcal Z}(z_{i-w}^{i-1})$.
    
    Let
    \[
        \alpha_{\mathcal Z}
        :=
        \mathbb E\!\left[
            |\Gamma_{\mathcal Z}(Z_0)|
        \right],
    \]
    and define the bad-span event
    \[
        B
        :=
        \left\{
            S(Z_{-w}^{-1})<w_s
        \right\}.
    \]
    Since the tokenizer is $(\epsilon,w,w_s)$-typical,
    \[
        \Pr(B)\le \epsilon.
    \]
    
    Define the pointwise token losses
    \[
        \ell_{\mathrm{typ},\eta}
        :=
        -\log_2
        \tilde q_{\mathrm{typ},\eta}(Z_0\mid Z_{-w}^{-1}),
    \]
    and
    \[
        \ell_{\mathrm{tr},\eta}
        :=
        -\log_2
        \tilde q_{\mathrm{tr},\eta}(Z_0\mid Z_{-w_s}^{-1}).
    \]
    On the good-span event $B^c$, the $w$-token context contains the required
    $w_s$ source-symbol suffix, so by construction
    \[
        \ell_{\mathrm{typ},\eta}
        =
        \ell_{\mathrm{tr},\eta}
        \qquad
        \text{on } B^c.
    \]
    On the bad-span event $B$, the predictor
    $\tilde q_{\mathrm{typ},\eta}$ is uniform over $\mathcal Z$, and therefore
    \[
        \ell_{\mathrm{typ},\eta}
        =
        \log_2|\mathcal Z|
        \qquad
        \text{on } B.
    \]
    
    Thus,
    \begin{align*}
    L_Z^{\mathrm{tok}}(\tilde q_{\mathrm{typ},\eta})
    &=
    \frac{1}{\alpha_{\mathcal Z}}
    \mathbb E\!\left[
        \ell_{\mathrm{typ},\eta}
    \right] \\
    &=
    \frac{1}{\alpha_{\mathcal Z}}
    \mathbb E\!\left[
        \ell_{\mathrm{typ},\eta}\mathbf 1_{B^c}
    \right]
    +
    \frac{1}{\alpha_{\mathcal Z}}
    \mathbb E\!\left[
        \ell_{\mathrm{typ},\eta}\mathbf 1_B
    \right] \\
    &=
    \frac{1}{\alpha_{\mathcal Z}}
    \mathbb E\!\left[
        \ell_{\mathrm{tr},\eta}\mathbf 1_{B^c}
    \right]
    +
    \frac{\Pr(B)\log_2|\mathcal Z|}{\alpha_{\mathcal Z}} \\
    &\le
    \frac{1}{\alpha_{\mathcal Z}}
    \mathbb E\!\left[
        \ell_{\mathrm{tr},\eta}
    \right]
    +
    \frac{\Pr(B)\log_2|\mathcal Z|}{\alpha_{\mathcal Z}} \\
    &=
    L_Z^{\mathrm{tok}}(\tilde q_{\mathrm{tr},\eta})
    +
    \frac{\Pr(B)\log_2|\mathcal Z|}{\alpha_{\mathcal Z}} \\
    &\le
    L_Y(q_\eta)
    +
    \epsilon\frac{\log_2|\mathcal Z|}{\alpha_{\mathcal Z}}.
    \end{align*}
    Since $\mathcal L_Z^{\mathrm{tok}}(w)$ is the minimum loss over
    $\mathcal Q_w^{\mathcal Z}$, we obtain
    \[
        \mathcal L_Z^{\mathrm{tok}}(w)
        \le
        L_Y(q_\eta)
        +
        \epsilon\frac{\log_2|\mathcal Z|}{\alpha_{\mathcal Z}}.
    \]
    
    It remains to let $\eta\downarrow 0$. For every context $c$ and symbol $y$,
    \[
        q_\eta(y\mid c)
        \to
        P_{w_s}(y\mid c).
    \]
    Moreover, on the support of $P_{w_s}(\cdot\mid c)$, for $\eta\le 1/2$,
    \[
        q_\eta(y\mid c)
        \ge
        (1-\eta)P_{w_s}(y\mid c)
        \ge
        \frac12 P_{w_s}(y\mid c).
    \]
    Hence
    \[
        -\log_2 q_\eta(y\mid c)
        \le
        -\log_2 P_{w_s}(y\mid c)+1
    \]
    on the support of $P_{w_s}$. Since the alphabets are finite, dominated
    convergence gives
    \[
        L_Y(q_\eta)
        \to
        \mathcal L_Y(w_s)
        \qquad
        \text{as }\eta\downarrow 0.
    \]
    Therefore,
    \[
        \mathcal L_Z^{\mathrm{tok}}(w)
        \le
        \mathcal L_Y(w_s)
        +
        \epsilon\frac{\log_2|\mathcal Z|}{\alpha_{\mathcal Z}}.
    \]
    Finally, using
    \[
        R_{\mathcal Z}
        =
        \frac{\log_2|\mathcal Z|}
        {\alpha_{\mathcal Z}\log_2|\mathcal Y|},
    \]
    we get
    \[
        \mathcal L_Z^{\mathrm{tok}}(w)
        \le
        \mathcal L_Y(w_s)
        +
        \epsilon\,R_{\mathcal Z}\log_2|\mathcal Y|.
    \]
    This proves the theorem.
    \end{proof}
    
    \newpage
    \section{Existence of Typical Source-Span Tokenizers}
    \label{app:typical-heavy-hitting-tokenizers}
    
    In this appendix, we show that the typical source-span condition in
    Definition~\ref{def:context_extending} is satisfied by a class of learned
    greedy tokenizers under a typical version of the heavy-hitting condition studied
    by \citet{rajaramanAnalysisTokenizationTransformers2024}. The purpose of this
    appendix is not to reprove their LZW learning result, but to translate
    high-probability substring coverage into our source-span language.
    
    The conclusion is that, under the same Markov positivity assumption used by
    \citet{rajaramanAnalysisTokenizationTransformers2024}, a learned
    budget-$d$ greedy tokenizer satisfying a typical high-probability-continuation
    coverage condition is
    \[
        \left(4\eta,\; w,\; w_d\right)\text{-typical},
        \qquad
        w_d
        :=
        \left\lfloor \frac{3}{4}w\ell_d \right\rfloor,
    \]
    where
    \[
        \ell_d
        :=
        \frac{\beta\log_2 d}{\log_2(1/\delta)}.
    \]
    Combining this with Theorem~\ref{thm:tokenization_gain} gives a finite-context
    loss guarantee for the learned tokenizer.
    
    \subsection{Setup}
    
    Let $\{Y_i\}_{i\in\mathbb Z}$ be a stationary ergodic $k$-th order Markov source
    over a finite alphabet $\mathcal Y$. We write its transition kernel as
    \[
        P(y\mid c),
        \qquad
        c\in\mathcal Y^k,\ y\in\mathcal Y,
    \]
    where $c$ denotes the previous $k$ source symbols. We assume the following
    positivity condition.
    
    \begin{assumption}[$\delta$-positive $k$-th order Markov source]
    \label{assump:hh-delta-positive}
    There exists $\delta>0$ such that
    \[
        \min_{c\in\mathcal Y^k,\ y\in\mathcal Y}P(y\mid c)\ge \delta .
    \]
    \end{assumption}
    
    \begin{remark}
    The first-order Markov setting considered by
    \citet{rajaramanAnalysisTokenizationTransformers2024} extends to the
    $k$-th order setting by treating the length-$k$ context $Y_{i-k}^{i-1}$ as the
    Markov state. We write the proof directly in $k$-th order notation to match the
    rest of the paper.
    \end{remark}
    
    For a string $u\in\mathcal Y^*$ with $|u|\ge k$, let
    $\operatorname{suf}_k(u)$ denote the length-$k$ suffix of $u$. For a string
    $t=t_1\cdots t_m\in\mathcal Y^m$ and an initial context
    $c\in\mathcal Y^k$, define
    \[
        P(t\mid c)
        :=
        \prod_{j=1}^{m}
        P\!\left(
            t_j
            \,\middle|\,
            \operatorname{suf}_k(c\,t_1^{j-1})
        \right),
    \]
    where $c\,t_1^{j-1}$ denotes concatenation of the initial context $c$ and the
    prefix $t_1^{j-1}$. Also define
    \[
        p_{\max}(t)
        :=
        \max_{c\in\mathcal Y^k}P(t\mid c).
    \]
    By Assumption~\ref{assump:hh-delta-positive}, every nonempty string $t$
    satisfies
    \begin{equation}
        p_{\max}(t)\ge \delta^{|t|}.
        \label{eq:pmax-lower-delta}
    \end{equation}
    Indeed, for every initial context $c$, the product defining $P(t\mid c)$
    contains $|t|$ transition probabilities, each at least $\delta$. This is the
    only source-probability estimate we need below.
    
    Let $\mathcal Z\subseteq\mathcal Y^*$ be a finite greedy vocabulary containing
    all single-symbol strings. As in the main text, $\Pi_{\mathcal Z}$ denotes the
    greedy parser and $\Gamma_{\mathcal Z}$ denotes the expansion map. We consider
    the stationary token process $\{Z_j\}$ induced by applying
    $\Pi_{\mathcal Z}$ to the stationary source. Recall that
    \[
        S(Z_{i-w}^{i-1})
        :=
        \left|\Gamma_{\mathcal Z}(Z_{i-w}^{i-1})\right|
        =
        \sum_{j=i-w}^{i-1}|\Gamma_{\mathcal Z}(Z_j)|
    \]
    is the source span of a $w$-token context.
    
    \subsection{High-probability substrings and typical coverage}
    
    For $\beta\in(0,1)$, define the high-probability substring set
    \[
        M_\beta
        :=
        \left\{
            t\in\mathcal Y^*:
            p_{\max}(t)\ge \frac{\delta}{d^\beta}
        \right\}.
    \]
    This is the high-probability region used in the LZW analysis of
    \citet{rajaramanAnalysisTokenizationTransformers2024}. In particular, their
    Lemma~A.10 shows that, under the corresponding Markov positivity assumption,
    LZW learns all substrings in $M_\beta$ with high probability for any fixed
    $\beta<1$. We use a weaker, typical version of this coverage property.
    
    Let $T=Z_0$ be a token drawn from the stationary emitted-token distribution.
    Let $A_0\in\mathcal Y$ denote the next source symbol immediately following
    the source string $\Gamma_{\mathcal Z}(T)$ in the underlying source sequence.
    Thus $\Gamma_{\mathcal Z}(T)A_0$ is the one-symbol continuation of the emitted
    token along the actual source realization.
    
    \begin{definition}[Typical coverage of high-probability continuations]
    \label{def:typical-Mbeta-coverage}
    A greedy vocabulary $\mathcal Z$ satisfies $(\beta,\eta)$-typical
    $M_\beta$-coverage if
    \[
        \Pr\!\left[
            \Gamma_{\mathcal Z}(T)A_0
            \in
            M_\beta\setminus \mathcal Z
        \right]
        \le
        \eta .
    \]
    \end{definition}
    
    This condition says that, at a stationary token boundary, the learned vocabulary
    usually contains the high-probability one-symbol continuation that would have
    allowed greedy parsing to extend the current token. It is the typical analogue
    of the deterministic event $M_\beta\subseteq\mathcal Z$ used in the LZW analysis
    of \citet{rajaramanAnalysisTokenizationTransformers2024}.
    
    \begin{assumption}[Typical LZW coverage]
    \label{assump:typical-lzw-coverage}
    Fix $\beta\in(0,1)$ and $\eta\ge 0$. Let $\mathcal Z_n$ be the dictionary
    learned from tokenizer-training data, with vocabulary budget at most $d$.
    For sufficiently large tokenizer-training data, there exists a failure
    probability $\xi_n(d,\beta,\eta,\delta)$, with
    \[
        \xi_n(d,\beta,\eta,\delta)\to 0
    \]
    as the tokenizer-training data grows, such that with probability at least
    $1-\xi_n(d,\beta,\eta,\delta)$, the learned dictionary $\mathcal Z_n$ satisfies
    $(\beta,\eta)$-typical $M_\beta$-coverage.
    \end{assumption}
    
    \begin{remark}
    Assumption~\ref{assump:typical-lzw-coverage} is weaker than the deterministic
    coverage guarantee proved by \citet{rajaramanAnalysisTokenizationTransformers2024}.
    Their Lemma~A.10 shows that LZW learns all substrings in $M_\beta$ with high
    probability. That deterministic event implies $(\beta,0)$-typical
    $M_\beta$-coverage. Here we state the weaker typical version because our theorem
    only needs the missed-continuation probability to be small, not zero.
    \end{remark}
    
    \subsection{Typical heavy-hitting}
    
    \begin{definition}[Typical heavy-hitting]
    \label{def:hh-typical-heavy-hitting}
    A greedy tokenizer $\mathcal Z$ is $(\beta,\eta)$-typically heavy-hitting if
    \[
        \Pr_T\!\left[
            p_{\max}\!\left(\Gamma_{\mathcal Z}(T)\right)>d^{-\beta}
        \right]
        \le
        \eta .
    \]
    \end{definition}
    
    \begin{lemma}[Typical coverage implies typical heavy-hitting]
    \label{lem:typical-coverage-to-typical-heavy}
    If $\mathcal Z$ satisfies $(\beta,\eta)$-typical $M_\beta$-coverage, then
    $\mathcal Z$ is $(\beta,\eta)$-typically heavy-hitting.
    \end{lemma}
    
    \begin{proof}
    Let $T=Z_0$ be a stationary emitted token, and let $A_0$ be the next source
    symbol after $\Gamma_{\mathcal Z}(T)$ in the underlying source sequence.
    
    Suppose that
    \[
        p_{\max}\!\left(\Gamma_{\mathcal Z}(T)\right)>d^{-\beta}.
    \]
    Then there exists a length-$k$ context $c^\star\in\mathcal Y^k$ such that
    \[
        P(\Gamma_{\mathcal Z}(T)\mid c^\star)>d^{-\beta}.
    \]
    Let
    \[
        c_T
        :=
        \operatorname{suf}_k\!\left(c^\star\Gamma_{\mathcal Z}(T)\right)
    \]
    be the length-$k$ context that would precede the next source symbol after
    generating $\Gamma_{\mathcal Z}(T)$ from the initial context $c^\star$. By
    Assumption~\ref{assump:hh-delta-positive},
    \[
        P(A_0\mid c_T)\ge \delta.
    \]
    Therefore,
    \[
        P(\Gamma_{\mathcal Z}(T)A_0\mid c^\star)
        =
        P(\Gamma_{\mathcal Z}(T)\mid c^\star)
        P(A_0\mid c_T)
        >
        \frac{\delta}{d^\beta}.
    \]
    Hence
    \[
        p_{\max}\!\left(\Gamma_{\mathcal Z}(T)A_0\right)
        >
        \frac{\delta}{d^\beta},
    \]
    and therefore
    \[
        \Gamma_{\mathcal Z}(T)A_0\in M_\beta.
    \]
    
    On the other hand, greedy parsing emitted $T$ rather than a longer token. Since
    $\Gamma_{\mathcal Z}(T)A_0$ is a longer string matching the current source
    position, it cannot belong to $\mathcal Z$; otherwise greedy parsing would have
    continued beyond $T$. Thus
    \[
        \Gamma_{\mathcal Z}(T)A_0\in M_\beta\setminus\mathcal Z.
    \]
    We have shown the event inclusion
    \[
        \left\{
            p_{\max}\!\left(\Gamma_{\mathcal Z}(T)\right)>d^{-\beta}
        \right\}
        \subseteq
        \left\{
            \Gamma_{\mathcal Z}(T)A_0\in M_\beta\setminus\mathcal Z
        \right\}.
    \]
    Taking probabilities and using Definition~\ref{def:typical-Mbeta-coverage}
    gives
    \[
        \Pr_T\!\left[
            p_{\max}\!\left(\Gamma_{\mathcal Z}(T)\right)>d^{-\beta}
        \right]
        \le
        \eta.
    \]
    Thus $\mathcal Z$ is $(\beta,\eta)$-typically heavy-hitting.
    \end{proof}
    
    \begin{corollary}[Existence of typically heavy-hitting learned tokenizers]
    \label{cor:exists-typical-heavy-hitting}
    Under Assumptions~\ref{assump:hh-delta-positive} and
    \ref{assump:typical-lzw-coverage}, with probability at least
    $1-\xi_n(d,\beta,\eta,\delta)$ over the tokenizer-training data, the learned
    greedy tokenizer $\mathcal Z_n$ is $(\beta,\eta)$-typically heavy-hitting.
    \end{corollary}
    
    \begin{proof}
    By Assumption~\ref{assump:typical-lzw-coverage}, with probability at least
    $1-\xi_n(d,\beta,\eta,\delta)$, the learned dictionary satisfies
    $(\beta,\eta)$-typical $M_\beta$-coverage. Lemma~\ref{lem:typical-coverage-to-typical-heavy}
    then implies that $\mathcal Z_n$ is $(\beta,\eta)$-typically heavy-hitting.
    \end{proof}
    
    \subsection{From typical heavy-hitting to typical source span}
    
    We now prove the part specific to our paper: typical heavy-hitting implies the
    typical source-span condition in Definition~\ref{def:context_extending}.
    
    \begin{lemma}[Typical heavy-hitting implies long tokens with high probability]
    \label{lem:hh-long-tokens}
    Assume $\mathcal Z$ is $(\beta,\eta)$-typically heavy-hitting. Define
    \[
        \ell_d
        :=
        \frac{\beta\log_2 d}{\log_2(1/\delta)}.
    \]
    Then
    \[
        \Pr_T\!\left[
            |\Gamma_{\mathcal Z}(T)|<\ell_d
        \right]
        \le \eta .
    \]
    \end{lemma}
    
    \begin{proof}
    Call a token $T$ good if
    \[
        p_{\max}\!\left(\Gamma_{\mathcal Z}(T)\right)\le d^{-\beta}.
    \]
    By Definition~\ref{def:hh-typical-heavy-hitting},
    \[
        \Pr_T[T\text{ is not good}]\le \eta.
    \]
    If $T$ is good, then by \eqref{eq:pmax-lower-delta},
    \[
        \delta^{|\Gamma_{\mathcal Z}(T)|}
        \le
        p_{\max}\!\left(\Gamma_{\mathcal Z}(T)\right)
        \le
        d^{-\beta}.
    \]
    Taking logarithms gives
    \[
        |\Gamma_{\mathcal Z}(T)|
        \ge
        \frac{\beta\log_2 d}{\log_2(1/\delta)}
        =
        \ell_d.
    \]
    Therefore the event
    \[
        |\Gamma_{\mathcal Z}(T)|<\ell_d
    \]
    can occur only when $T$ is not good. Hence
    \[
        \Pr_T\!\left[
            |\Gamma_{\mathcal Z}(T)|<\ell_d
        \right]
        \le \eta.
    \]
    \end{proof}
    
    \begin{lemma}[Typical heavy-hitting implies typical source span]
    \label{lem:hh-typical-source-span}
    Assume $\mathcal Z$ is $(\beta,\eta)$-typically heavy-hitting. Fix a
    token-window length $w$, and define
    \[
        \ell_d
        :=
        \frac{\beta\log_2 d}{\log_2(1/\delta)}.
    \]
    Let
    \[
        w_d
        :=
        \left\lfloor \frac{3}{4}w\ell_d \right\rfloor .
    \]
    Then
    \[
        \Pr\!\left[
            S(Z_{i-w}^{i-1}) < w_d
        \right]
        \le
        4\eta.
    \]
    Equivalently, $\Pi_{\mathcal Z}$ is
    \[
        \left(4\eta,\; w,\; w_d\right)\text{-typical}.
    \]
    \end{lemma}
    
    \begin{proof}
    For each token $Z_j$, define the bad-token indicator
    \[
        B_j
        :=
        \mathbf 1\left\{
            |\Gamma_{\mathcal Z}(Z_j)|<\ell_d
        \right\}.
    \]
    By Lemma~\ref{lem:hh-long-tokens},
    \[
        \mathbb E[B_j]=\Pr(B_j=1)\le \eta.
    \]
    For the window $Z_{i-w}^{i-1}$, let
    \[
        B
        :=
        \sum_{j=i-w}^{i-1}B_j
    \]
    be the number of bad tokens in the window. By linearity of expectation,
    \[
        \mathbb E[B]\le w\eta.
    \]
    
    Every good token has source length at least $\ell_d$, and every bad token has
    source length at least $1$. Therefore,
    \[
        S(Z_{i-w}^{i-1})
        =
        \sum_{j=i-w}^{i-1}|\Gamma_{\mathcal Z}(Z_j)|
        \ge
        (w-B)\ell_d + B.
    \]
    In particular, if $B\le w/4$, then
    \[
        S(Z_{i-w}^{i-1})
        \ge
        \left(w-\frac{w}{4}\right)\ell_d
        =
        \frac{3}{4}w\ell_d
        \ge
        w_d.
    \]
    Thus,
    \[
        \left\{
            S(Z_{i-w}^{i-1}) < w_d
        \right\}
        \subseteq
        \left\{
            B>\frac{w}{4}
        \right\}.
    \]
    By Markov's inequality,
    \[
        \Pr\!\left[
            B>\frac{w}{4}
        \right]
        \le
        \frac{\mathbb E[B]}{w/4}
        \le
        4\eta.
    \]
    Therefore,
    \[
        \Pr\!\left[
            S(Z_{i-w}^{i-1}) < w_d
        \right]
        \le
        4\eta.
    \]
    This is exactly the $\left(4\eta,w,w_d\right)$-typicality condition.
    \end{proof}
    
    \subsection{Tokenizer rate bound}
    
    We also need to control the rate term in Theorem~\ref{thm:tokenization_gain}.
    Recall that
    \[
        \alpha_{\mathcal Z}
        :=
        \mathbb E\!\left[
            |\Gamma_{\mathcal Z}(Z_0)|
        \right]
    \]
    and
    \[
        R_{\mathcal Z}
        :=
        \frac{\log_2|\mathcal Z|}
             {\alpha_{\mathcal Z}\log_2|\mathcal Y|}.
    \]
    
    \begin{lemma}[Rate bound for typically heavy-hitting tokenizers]
    \label{lem:hh-rate-bound}
    Assume $\mathcal Z$ is $(\beta,\eta)$-typically heavy-hitting and
    $|\mathcal Z|\le d$. Then
    \[
        \alpha_{\mathcal Z}
        \ge
        (1-\eta)\ell_d+\eta,
    \]
    where
    \[
        \ell_d
        :=
        \frac{\beta\log_2 d}{\log_2(1/\delta)}.
    \]
    Consequently,
    \[
        R_{\mathcal Z}
        \le
        \frac{\log_2 d}
             {\big((1-\eta)\ell_d+\eta\big)\log_2|\mathcal Y|}.
    \]
    \end{lemma}
    
    \begin{proof}
    By Lemma~\ref{lem:hh-long-tokens}, with probability at least $1-\eta$, an
    emitted token has source length at least $\ell_d$. On the remaining event, the
    token has source length at least $1$, since all tokens are nonempty source
    strings. Hence
    \[
        \alpha_{\mathcal Z}
        =
        \mathbb E\!\left[
            |\Gamma_{\mathcal Z}(Z_0)|
        \right]
        \ge
        (1-\eta)\ell_d+\eta.
    \]
    Since $|\mathcal Z|\le d$,
    \[
        R_{\mathcal Z}
        =
        \frac{\log_2|\mathcal Z|}
             {\alpha_{\mathcal Z}\log_2|\mathcal Y|}
        \le
        \frac{\log_2 d}
             {\big((1-\eta)\ell_d+\eta\big)\log_2|\mathcal Y|}.
    \]
    \end{proof}
    
    \subsection{Main result}
    
    We now combine the existence of a typically heavy-hitting dictionary with the
    source-span and rate bounds above.
    
    \begin{theorem}[Existence of tokenizers with typical effective context]
    \label{thm:exists-typical-effective-context}
    Assume $\{Y_i\}$ satisfies Assumption~\ref{assump:hh-delta-positive}. Fix
    $\beta\in(0,1)$, vocabulary budget $d$, token-window length $w$, and
    $\eta\in[0,1/4]$. Define
    \[
        \ell_d
        :=
        \frac{\beta\log_2 d}{\log_2(1/\delta)}
    \]
    and
    \[
        w_d
        :=
        \left\lfloor \frac{3}{4}w\ell_d \right\rfloor .
    \]
    For sufficiently large tokenizer-training data, with probability at least
    $1-\xi_n(d,\beta,\eta,\delta)$ over the training sequence, the learned LZW
    tokenizer $\Pi_{\mathcal Z_n}$ is
    \[
        \left(4\eta,\; w,\; w_d\right)\text{-typical}.
    \]
    Moreover,
    \[
        R_{\mathcal Z_n}
        \le
        \frac{\log_2 d}
             {\big((1-\eta)\ell_d+\eta\big)\log_2|\mathcal Y|}.
    \]
    Consequently,
    \[
        \mathcal L_Z^{\mathrm{tok}}(w)
        \le
        \mathcal L_Y(w_d)
        +
        4\eta\,R_{\mathcal Z_n}\log_2|\mathcal Y|.
    \]
    In particular,
    \[
        \mathcal L_Z^{\mathrm{tok}}(w)
        \le
        \mathcal L_Y(w_d)
        +
        4\eta
        \frac{\log_2 d}
             {(1-\eta)\ell_d+\eta}.
    \]
    \end{theorem}
    
    \begin{proof}
    By Corollary~\ref{cor:exists-typical-heavy-hitting}, with probability at least
    $1-\xi_n(d,\beta,\eta,\delta)$, the learned LZW dictionary $\mathcal Z_n$ is
    $(\beta,\eta)$-typically heavy-hitting. On this event,
    Lemma~\ref{lem:hh-typical-source-span} gives
    \[
        \Pr\!\left[
            S(Z_{i-w}^{i-1}) < w_d
        \right]
        \le
        4\eta.
    \]
    Therefore $\Pi_{\mathcal Z_n}$ is
    \[
        \left(4\eta,\; w,\; w_d\right)\text{-typical}.
    \]
    
    Also, by Lemma~\ref{lem:hh-rate-bound},
    \[
        R_{\mathcal Z_n}
        \le
        \frac{\log_2 d}
             {\big((1-\eta)\ell_d+\eta\big)\log_2|\mathcal Y|}.
    \]
    
    Now apply Theorem~\ref{thm:tokenization_gain} with
    \[
        \epsilon=4\eta
        \qquad\text{and}\qquad
        w_s=w_d.
    \]
    This gives
    \[
        \mathcal L_Z^{\mathrm{tok}}(w)
        \le
        \mathcal L_Y(w_d)
        +
        4\eta\,R_{\mathcal Z_n}\log_2|\mathcal Y|.
    \]
    Substituting the rate bound yields
    \[
        \mathcal L_Z^{\mathrm{tok}}(w)
        \le
        \mathcal L_Y(w_d)
        +
        4\eta
        \frac{\log_2 d}
             {(1-\eta)\ell_d+\eta}.
    \]
    \end{proof}
    
    \begin{remark}
    The LZW guarantee of \citet{rajaramanAnalysisTokenizationTransformers2024}
    gives a stronger deterministic sufficient condition: with high probability,
    LZW learns all substrings in $M_\beta$, which implies $(\beta,0)$-typical
    $M_\beta$-coverage. We state the result using the weaker typical-coverage
    condition because Theorem~\ref{thm:tokenization_gain} is naturally a
    typical-span theorem: it only requires the window-span failure probability to be
    small, not zero.
    \end{remark}

    \newpage

    \section{Implementation Details}
\label{sec:implementation}

All experiments are implemented in PyTorch and run
on a single NVIDIA GPU (CUDA~12.8).
The full code, configs, and scripts needed to reproduce every figure in the
paper are available at the following anonymous repository:
\begin{center}
  \texttt{https://github.com/Amiroo23jf/bytes-are-not-enough}
\end{center}

\subsection{Transformer Architecture}
\label{app:impl:arch}

Both the fragmentation and tokenization experiments use the same causal
transformer architecture, differing only in embedding dimension and
feedforward width.
The model is a standard pre-norm transformer with
sinusoidal positional encodings, GELU
activations, and no weight tying between the
input embedding and the output projection.
All linear layers are bias-free; attention is implemented with
\texttt{scaled\_dot\_product\_attention}.
Weights are initialised from $\mathcal{N}(0,\,0.02^{2})$.

\begin{table}[h]
\centering
\caption{Transformer hyperparameters for each experiment family.}
\label{tab:arch}
\small
\begin{tabular}{lcc}
\toprule
\textbf{Hyperparameter} & \textbf{Fragmentation Transformer} & \textbf{Tokenization Transformer} \\
\midrule
Layers               & 4       & 4       \\
Embedding dim $d$    & 256     & 128     \\
FFN width            & 1024    & 512     \\
Attention heads      & 4       & 4       \\
Head dim             & 64      & 32      \\
Positional encoding  & Sinusoidal & Sinusoidal \\
Context window $W$   & $K{+}1$ & see \S\ref{app:impl:tok} \\
Vocab size $V$       & $2^M$   & $V \in \{2,4,6,8,10,15,20\}$ \\
Parameters (approx.) & 3.15M   & 789–794K \\
\bottomrule
\end{tabular}
\end{table}

\subsection{Fragmentation Transformer Experiment}
\label{app:impl:frag}

\paragraph{Source.}
We generate a binary Markov source of order $K=2$ with fragment alphabet
size $2^M = 8$ ($M=3$).
The transition matrix $P$ is drawn from a symmetric Dirichlet prior with
concentration $\alpha=0.5$.
The source sequence length is $n=2\times10^6$ symbols per run.

\paragraph{Context window.}
We set $W = K+1 = 3$.
With this choice the oracle-residual term in Theorem~1 vanishes, so the
gap between the fragment-level and source-level Bayes-optimal losses equals
exactly the phase mutual-information term.

\paragraph{Training.}
We train both the $Y$-model (fragment-level) and the $X$-model
(source-level) independently.
Each model is trained for 40\,000 gradient steps with the AdamW
optimiser~\citep{loshchilov2019decoupled} using the hyperparameters in
Table~\ref{tab:train_frag}.
Loss curves are smoothed with a trailing window of 500 steps for plotting.
All results are averaged over 3 independent seeds; shaded bands show
$\pm 1$ standard deviation.

\begin{table}[h]
\centering
\caption{Training hyperparameters for the fragmentation transformer experiment.}
\label{tab:train_frag}
\small
\begin{tabular}{lc}
\toprule
\textbf{Hyperparameter} & \textbf{Value} \\
\midrule
Optimiser            & AdamW \\
Learning rate        & $3 \times 10^{-4}$ \\
$(\beta_1, \beta_2)$ & $(0.9,\; 0.95)$ \\
Weight decay         & $0.01$ \\
Gradient clipping    & $1.0$ \\
Warmup steps         & $1$ \\
Batch size           & $128$ \\
Iterations           & $40\,000$ \\
Seeds                & 3 \;($\{0,1,2\}$) \\
\bottomrule
\end{tabular}
\end{table}

\subsection{Tokenization Transformer Experiment}
\label{app:impl:tok}

\paragraph{Source.}
We generate a binary Markov source of order $K=12$ with Dirichlet
concentration $\alpha=0.4$ and source length $n=2.5\times10^7$ bits.
One sequence is generated per seed and shared across all $(V,W)$
combinations within that seed.

\paragraph{BPE tokenizer.}
For each vocabulary size $V>2$ we train a greedy BPE tokenizer on the
first $5\times10^5$ bits of the source sequence.
For $V=2$ the source is used directly without tokenisation (ratio~$=1$).
The compression ratios achieved for the full source sequence are reported
in Table~\ref{tab:ratios}.

\begin{table}[h]
\centering
\caption{BPE compression ratios $|Y|/|Z|$ (seed 0, $K=12$, $\alpha=0.4$,
$n=2.5\times10^7$).}
\label{tab:ratios}
\small
\begin{tabular}{ccccccc}
\toprule
$V=4$ & $V=6$ & $V=8$ & $V=10$ & $V=15$ & $V=20$ \\
\midrule
1.590 & 2.071 & 2.546 & 2.944 & 3.614 & 4.172 \\
\bottomrule
\end{tabular}
\end{table}

\paragraph{Sweep.}
We sweep vocabulary sizes $V \in \{2,4,6,8,10,15,20\}$ and context window
$W \in \{4,6,12\}$, training one model per $(V,W,\text{seed})$ triple.
All models share the architecture in Table~\ref{tab:arch}.
Loss is reported in bits per source symbol: the raw cross-entropy is
divided by the token-level compression ratio so that all vocabulary sizes
are on a common scale.

\paragraph{Training.}
Training hyperparameters are given in Table~\ref{tab:train_tok}.
Results are averaged over 3 seeds; shaded bands show $\pm 1$ standard
deviation.

\begin{table}[h]
\centering
\caption{Training hyperparameters for the tokenization transformer experiment.}
\label{tab:train_tok}
\small
\begin{tabular}{lc}
\toprule
\textbf{Hyperparameter} & \textbf{Value} \\
\midrule
Optimiser            & AdamW \\
Learning rate        & $3 \times 10^{-4}$ \\
$(\beta_1, \beta_2)$ & $(0.9,\; 0.95)$ \\
Weight decay         & $0.01$ \\
Gradient clipping    & $1.0$ \\
Warmup steps         & $500$ \\
Batch size           & $128$ \\
Iterations           & $40\,000$ \\
Seeds                & 3 \;($\{0,1,2\}$) \\
\bottomrule
\end{tabular}
\end{table}

\subsection{Fragmentation N-gram Experiment}
\label{app:impl:ngram}

We verify Theorem~3.1 empirically with a $k$-th order $n$-gram model on a
binary Markov source.
For each pair $(k, M)$ in
$\{(1,2),(1,3),(1,4),(2,2),(2,3),(3,2)\}$
we sample 8 transition matrices independently from a symmetric Dirichlet
prior ($\alpha=0.5$) and generate a source sequence of $n=5\times10^5$
symbols per matrix.
We fit two $n$-gram models: one with context length $K=k$ (exact match)
and one with $K=k+1$ (one extra token of context).
Counts are smoothed with Laplace smoothing ($\alpha = 0.5$).
The theoretical excess loss from Theorem~3.1 is computed analytically
from the sampled transition matrix and compared to the empirical $n$-gram
loss.

\subsection{Span-Distribution Analysis}
\label{app:impl:span}

\paragraph{WikiText-103.}
We measure $P(\mathrm{span}(Z^W) = w_s)$—the distribution of source
characters covered by $W$ consecutive tokens—for the following
tokenizers trained on WikiText-103~\citep{merity2017pointer}: character,
BPE with $|V| \in \{1\mathrm{k},4\mathrm{k},16\mathrm{k},32\mathrm{k}\}$,
WordPiece (32k), SentencePiece (32k), BERT~\citep{devlinBERTPretrainingDeep2019}, and
GPT-2~\citep{radford2019language}.
We sweep $W \in \{1,2,4,8,16,32,64,128,256\}$.

\paragraph{Markov source.}
We repeat the same analysis on the synthetic binary Markov source used in
the tokenization transformer experiment ($K=12$, $\alpha=0.4$,
$n=2.5\times10^7$, seed~0), using the BPE tokenizers trained for that
experiment ($V \in \{2,4,6,8,10,15,20\}$).
Here the span $w_s$ is measured in source bits rather than characters, and
we sweep $W \in \{1,\ldots,12\}$.
The plotted quantity is the scaled CDF
$\varepsilon(w_s) \cdot \log_2 V \cdot \mathbb{E}[|Z|/|Y|]$,
which provides a normalised upper bound comparable across tokenizers.
    
     \section{Additional Experiments} \label{sec:additionalsynth}

   \subsection{Fragmentation Decomposition}
\label{app:fragmentation_decomposition}

We empirically verify the decomposition in Theorem~\ref{thm:fragmentation}
using finite-order Markov sources under several parameter settings. For each
configuration, we compare the optimal source-level loss with the normalized
loss after fragmentation, and then compare the observed fragmentation penalty
with the theoretical decomposition into phase ambiguity and context deficit.

Figure~\ref{fig:fragmentation_decomposition_all} shows the results across four
source regimes. In each row, the left panel reports the per-source-symbol loss:
the theoretical source entropy, the empirical $n$-gram loss on the original
source sequence, and the normalized $n$-gram loss on the fragmented sequence.
The right panel reports the penalty decomposition. The empirical penalty
matches the predicted total penalty, showing that the excess loss observed after
fragmentation is explained by the information-theoretic terms in
Theorem~\ref{thm:fragmentation}.

\begin{figure}[htbp]
    \centering

    \begin{subfigure}[t]{0.49\linewidth}
        \centering
        \includegraphics[width=\linewidth]{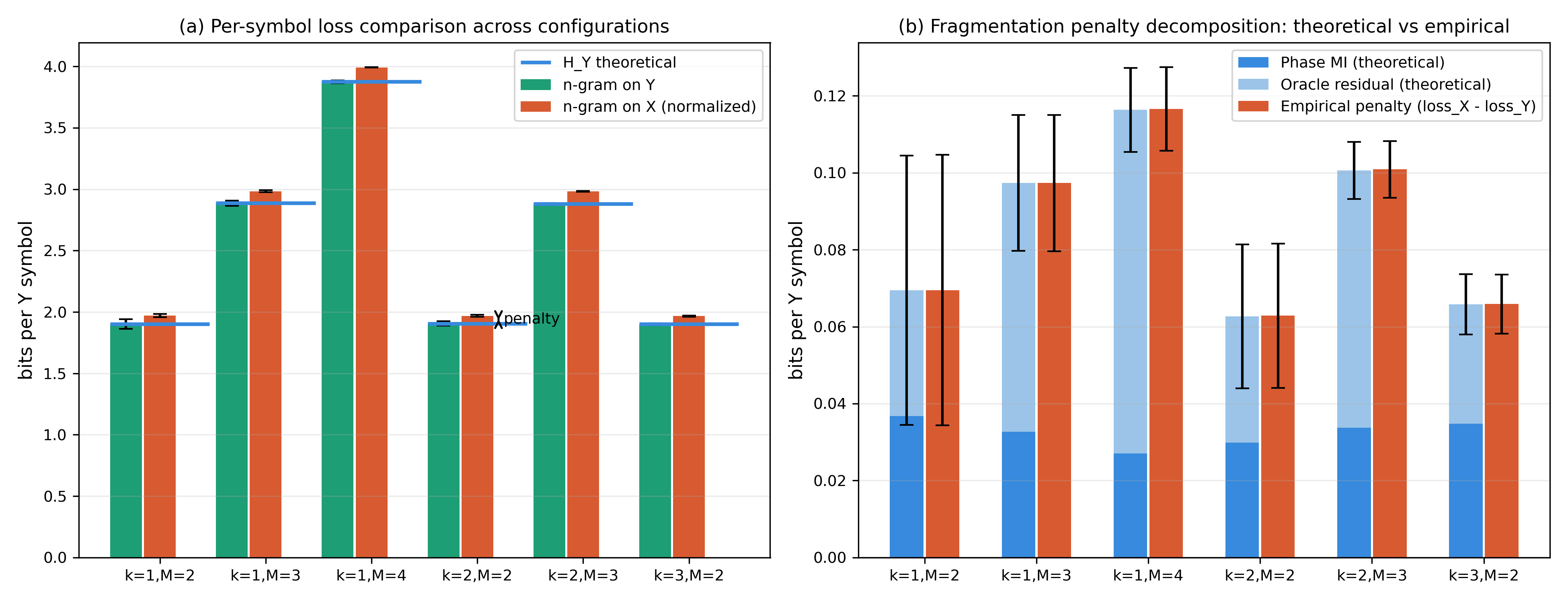}
        \caption{$K=0$, $\alpha=5$.}
        \label{fig:frag_decomp_k0_a5}
    \end{subfigure}
    \hfill
    \begin{subfigure}[t]{0.49\linewidth}
        \centering
        \includegraphics[width=\linewidth]{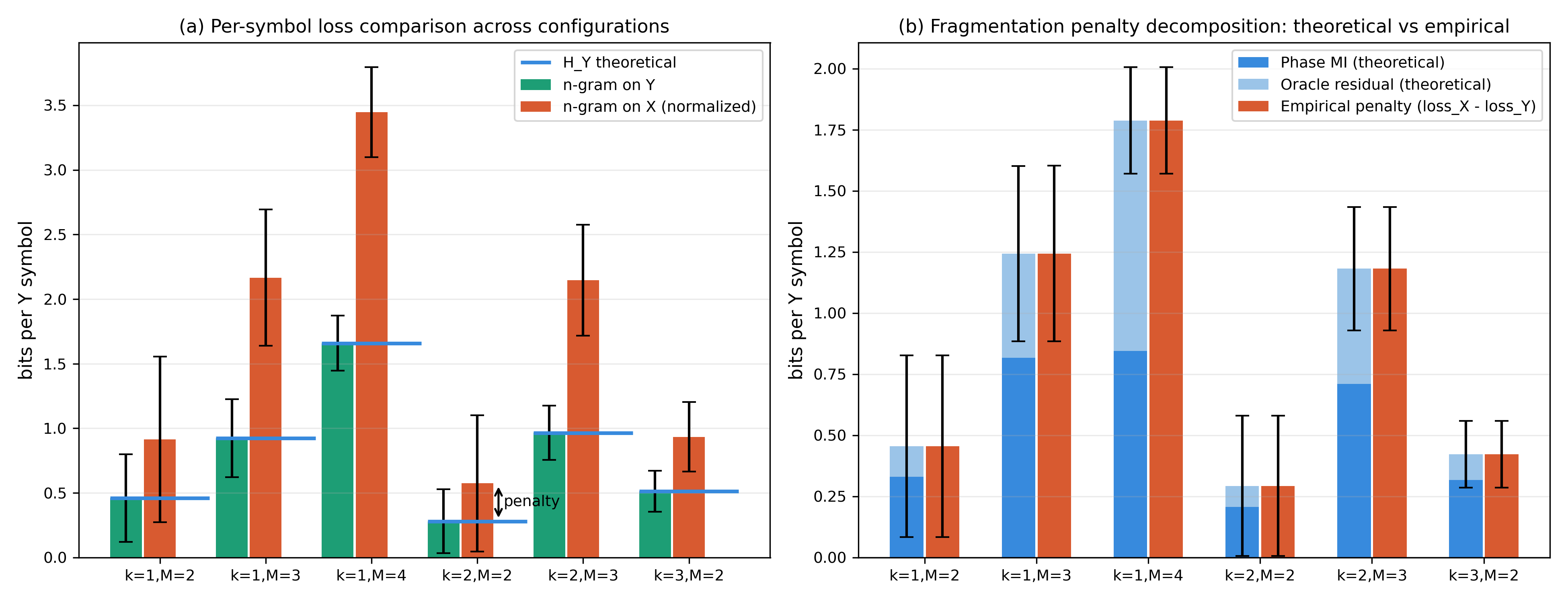}
        \caption{$K=0$, $\alpha=0.1$.}
        \label{fig:frag_decomp_k0_a0p1}
    \end{subfigure}

    \vspace{0.8em}

    \begin{subfigure}[t]{0.49\linewidth}
        \centering
        \includegraphics[width=\linewidth]{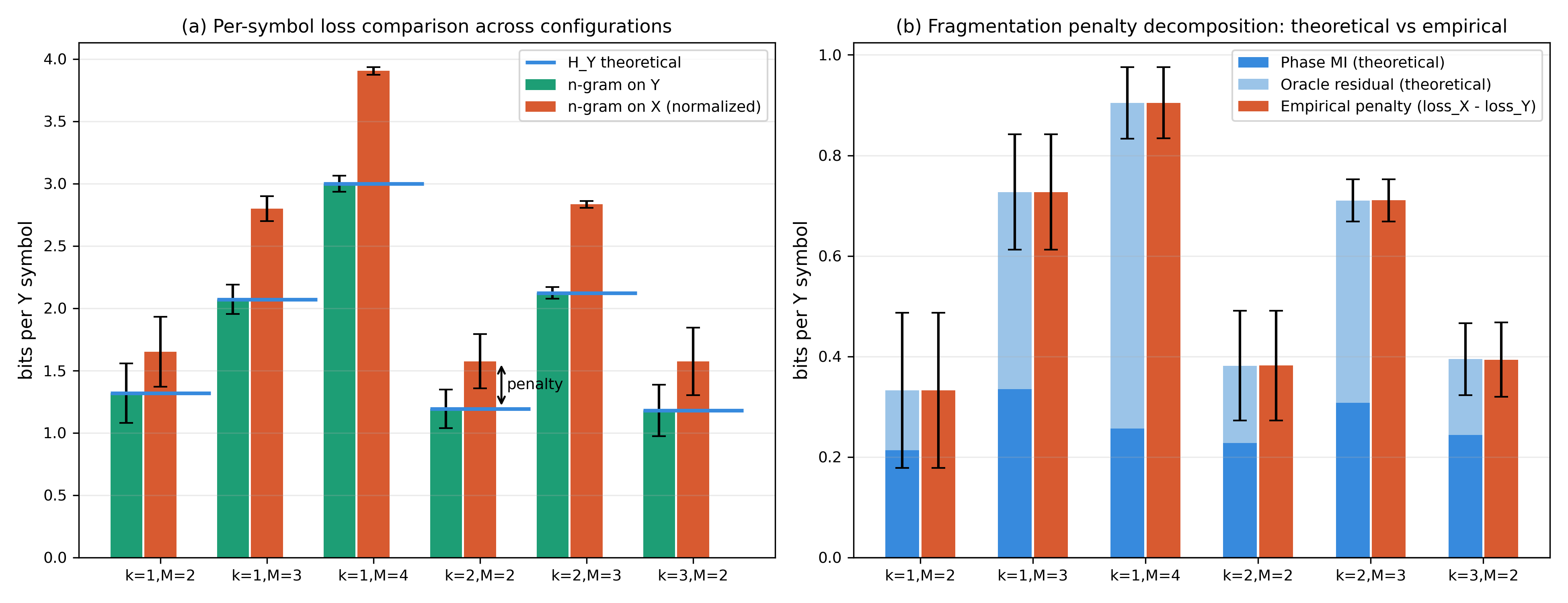}
        \caption{$K=0$, $\alpha=0.5$.}
        \label{fig:frag_decomp_k0_a0p5}
    \end{subfigure}
    \hfill
    \begin{subfigure}[t]{0.49\linewidth}
        \centering
        \includegraphics[width=\linewidth]{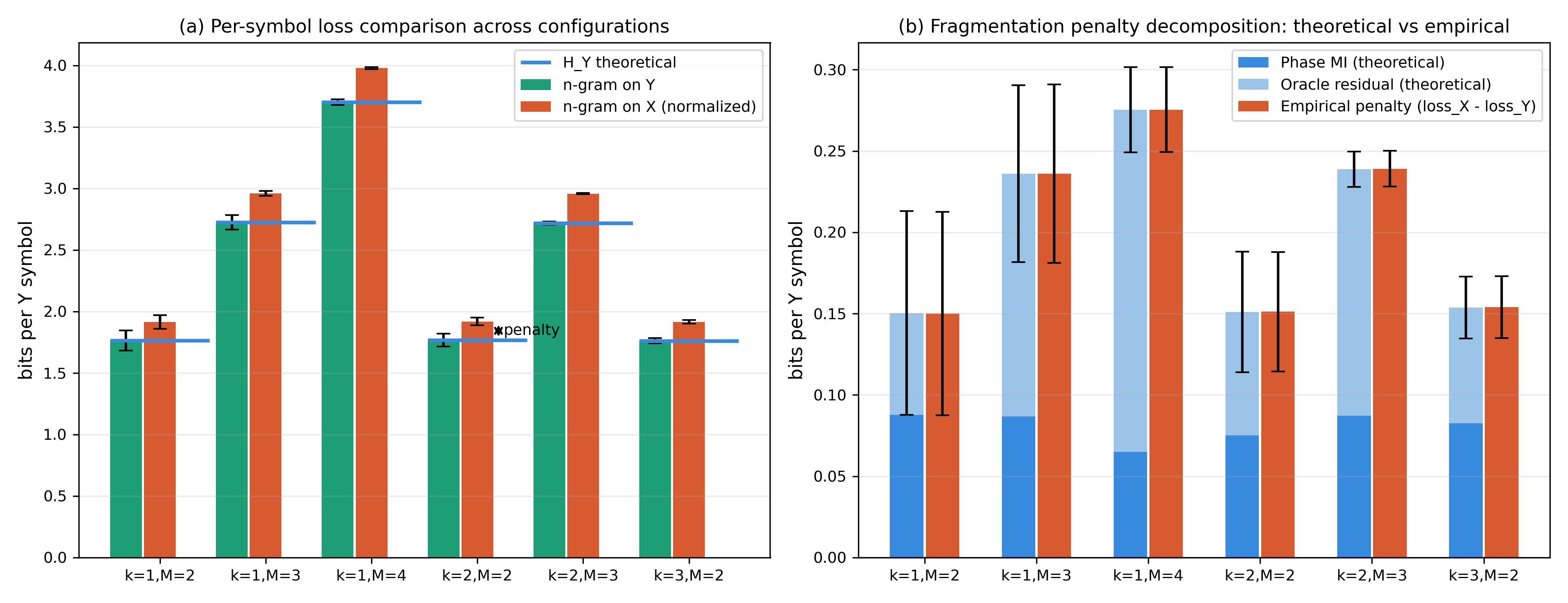}
        \caption{$K=0$, $\alpha=2$.}
        \label{fig:frag_decomp_k0_a2}
    \end{subfigure}

    \caption{
    Empirical verification of the fragmentation decomposition in
    Theorem~\ref{thm:fragmentation}. Each subfigure contains two panels. Left:
    per-source-symbol loss comparison between the theoretical source entropy,
    the $n$-gram loss on the original source sequence, and the normalized
    $n$-gram loss on the fragmented sequence. Right: comparison between the
    theoretical fragmentation penalty decomposition and the empirical penalty
    $\mathcal L_X^{\mathrm{frag}}-\mathcal L_Y$. Across settings, the empirical
    penalty closely follows the theoretical sum of phase ambiguity and context
    deficit.
    }
    \label{fig:fragmentation_decomposition_all}
\end{figure}


\subsection{Matched-Plateau Comparisons}
\label{app:matched_plateau}

Figure~\ref{fig:matched_plateau_curves} compares vocabulary/window pairs that
reach approximately the same plateau loss. The tokenized configurations achieve
this plateau with much smaller token windows than the binary baseline, and they
also reach it in fewer training iterations. This suggests that, in this
synthetic setting, tokenization improves both effective context and convergence
speed.

\begin{figure}[t]
    \centering
    \includegraphics[width=0.72\linewidth]{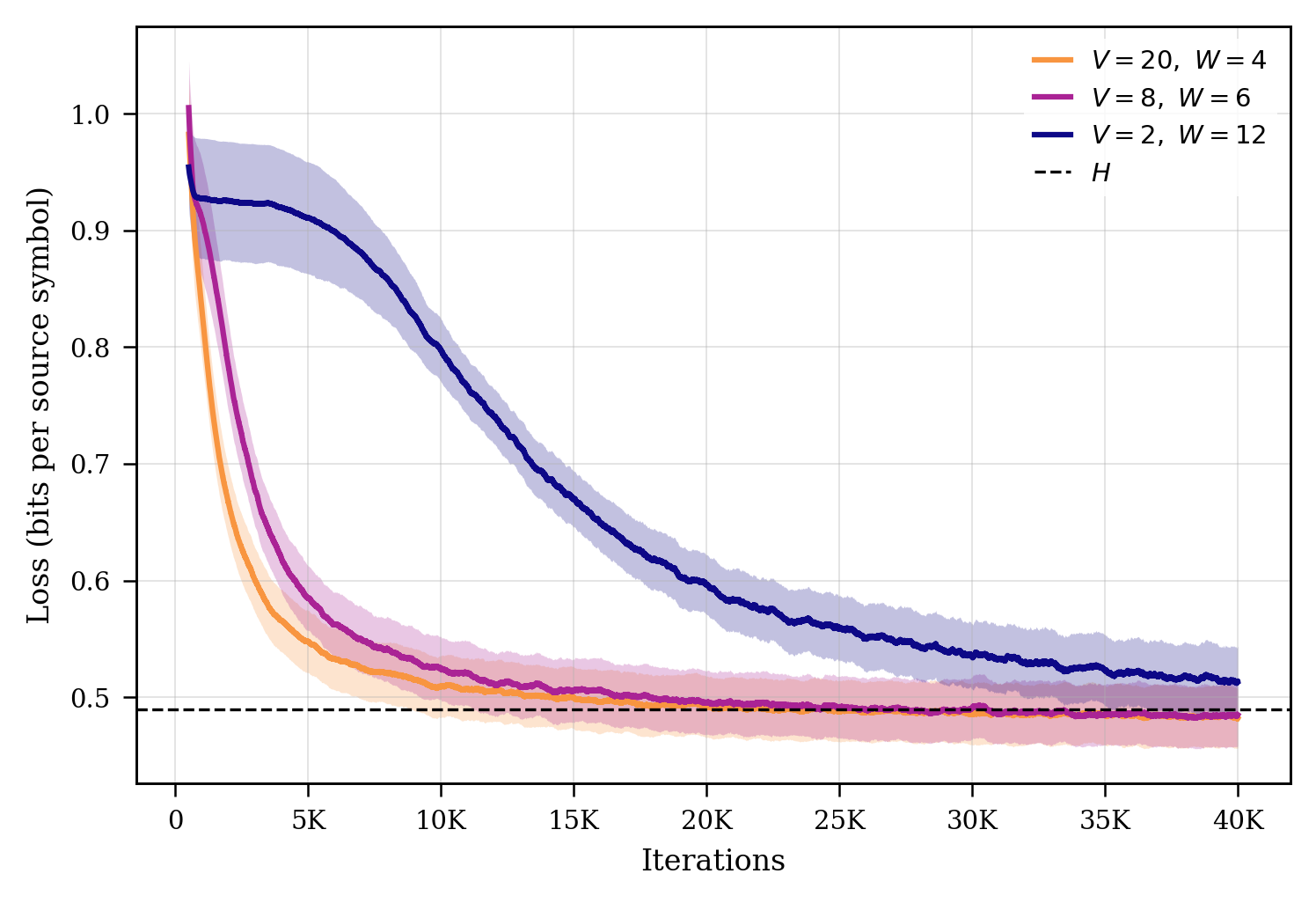}
    \caption{
    Matched-plateau comparison across vocabulary/window pairs on the synthetic
    Markov source. Larger-vocabulary tokenized models reach approximately the
    same entropy-rate plateau as the binary baseline using smaller token
    windows, and converge faster in training iterations. Shaded regions show
    variability across runs.
    }
    \label{fig:matched_plateau_curves}
\end{figure}

\subsection{Tokenizer Source-Span Distributions Across Window Sizes}
\label{app:cdf_extra}

We also evaluate the source-span behavior of real tokenizers on WikiText across
several token-window lengths. For each tokenizer and window size $w$, we
estimate the distribution of the source span $S(Z_1^w)$ and the corresponding
slack term from Theorem~\ref{thm:tokenization_gain},
\[
    \epsilon(w,w_s)\frac{\log_2|\mathcal Z|}{\alpha_{\mathcal Z}},
    \qquad
    \epsilon(w,w_s)
    =
    \Pr[S(Z_1^w)<w_s].
\]
Figures~\ref{fig:cdf_scaled_all_windows} and~\ref{fig:cdf_all_windows} show that
the same qualitative behavior persists across window sizes: tokenized windows
reliably span substantially more source characters than the raw character
window, and the slack remains small until the requested source span grows well
beyond $w$.

\begin{figure}[t]
    \centering

    \begin{subfigure}[t]{0.245\linewidth}
        \centering
        \includegraphics[width=\linewidth]{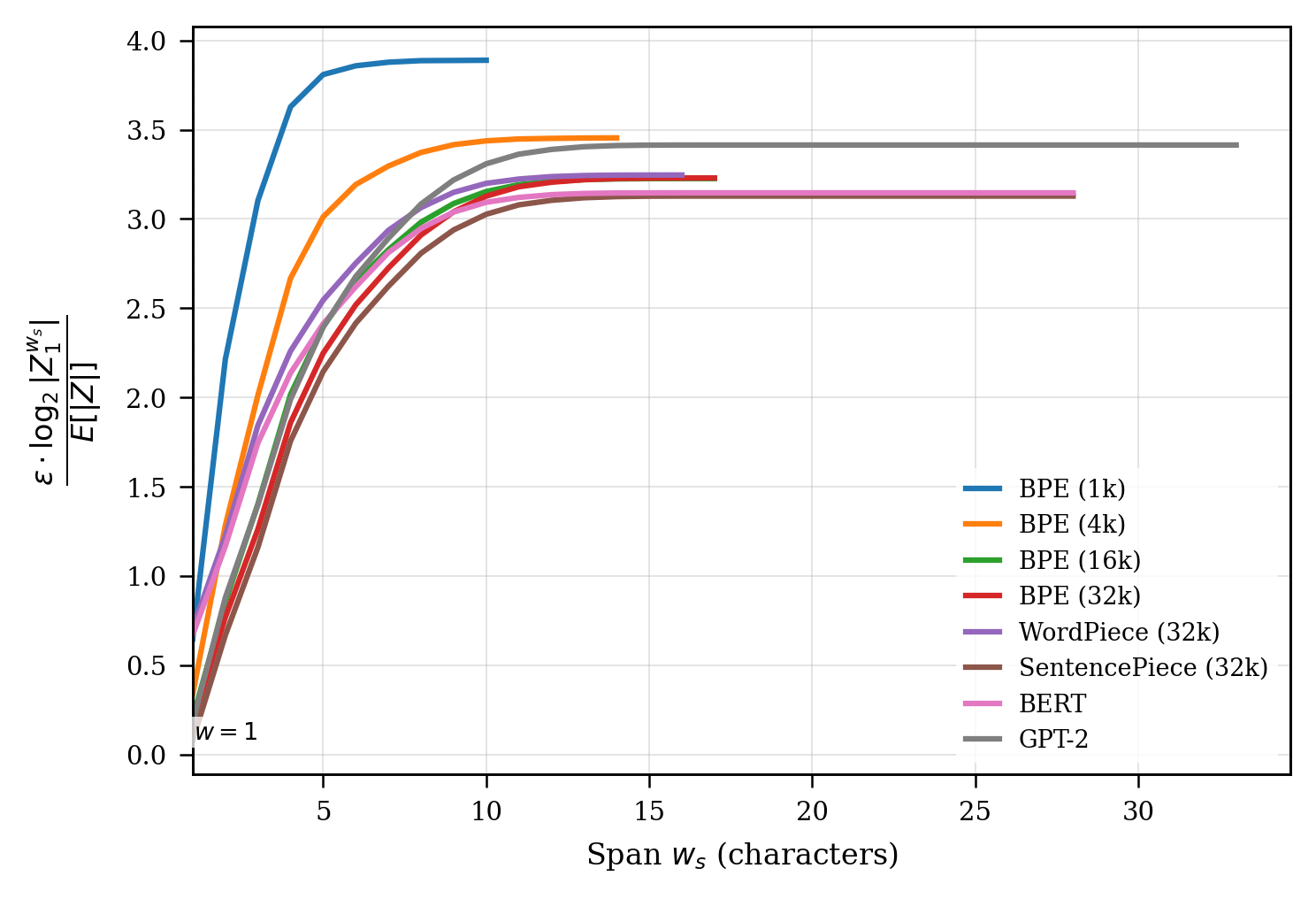}
        \caption{$w=1$}
    \end{subfigure}
    \begin{subfigure}[t]{0.245\linewidth}
        \centering
        \includegraphics[width=\linewidth]{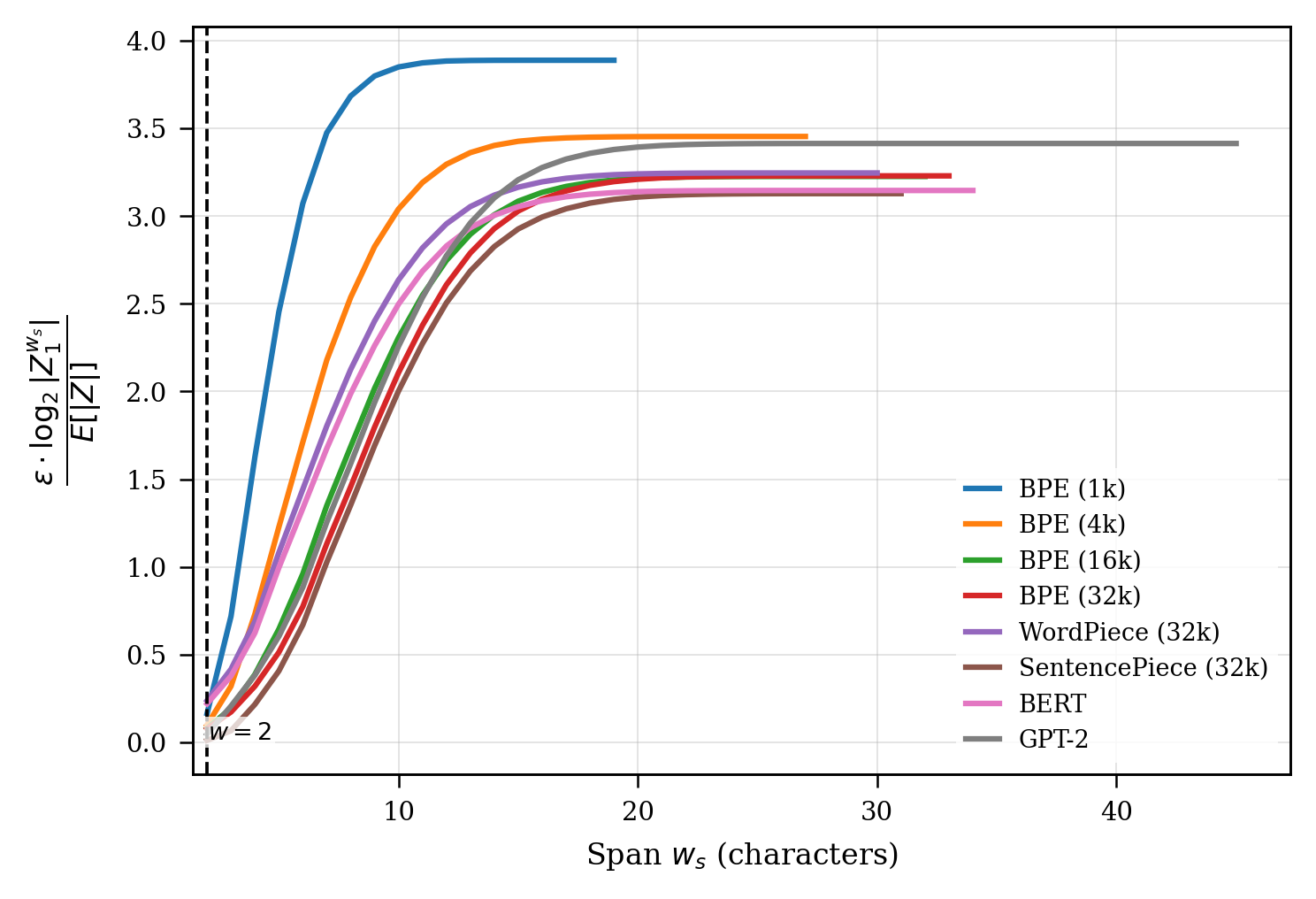}
        \caption{$w=2$}
    \end{subfigure}
    \begin{subfigure}[t]{0.245\linewidth}
        \centering
        \includegraphics[width=\linewidth]{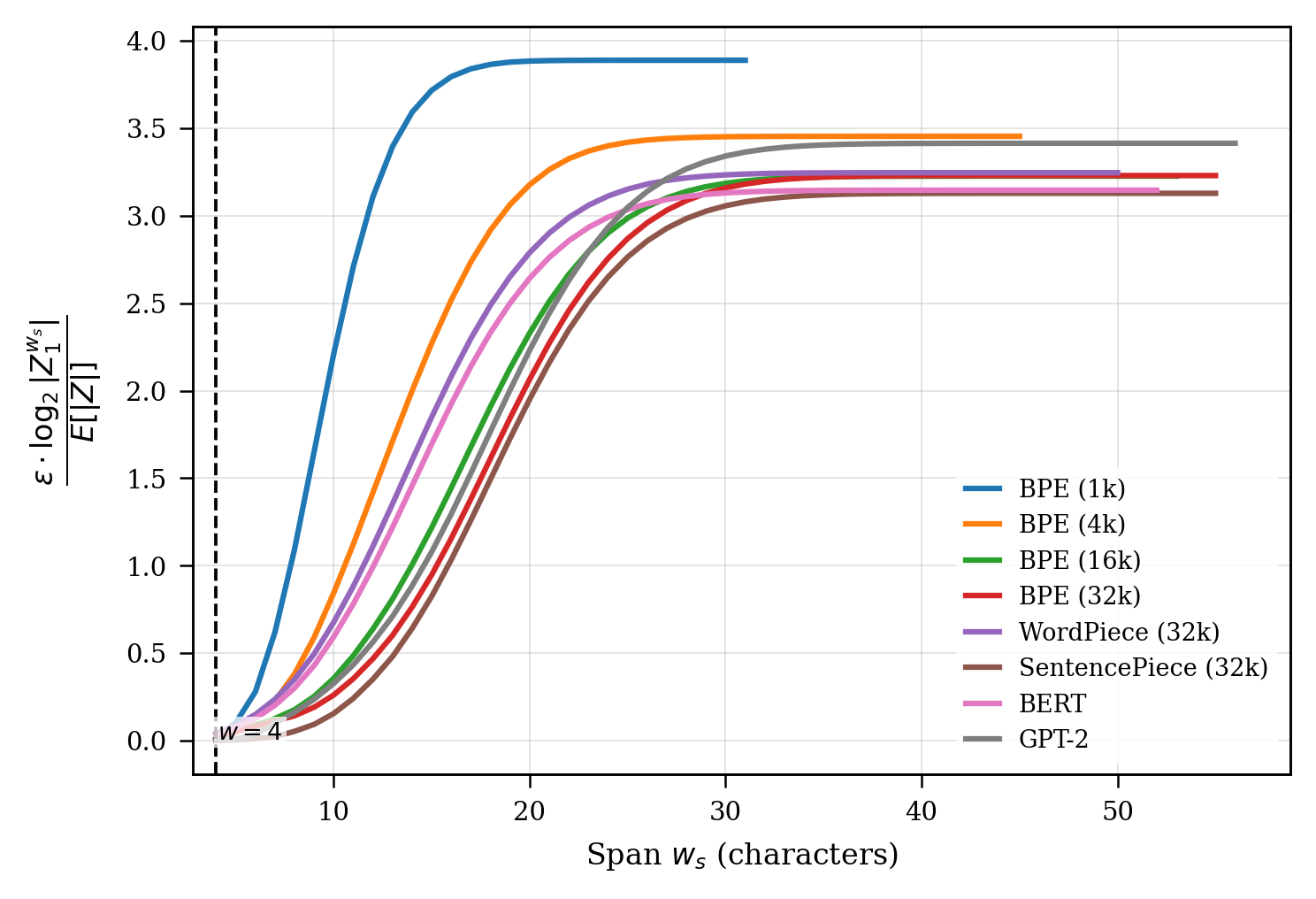}
        \caption{$w=4$}
    \end{subfigure}
    \begin{subfigure}[t]{0.245\linewidth}
        \centering
        \includegraphics[width=\linewidth]{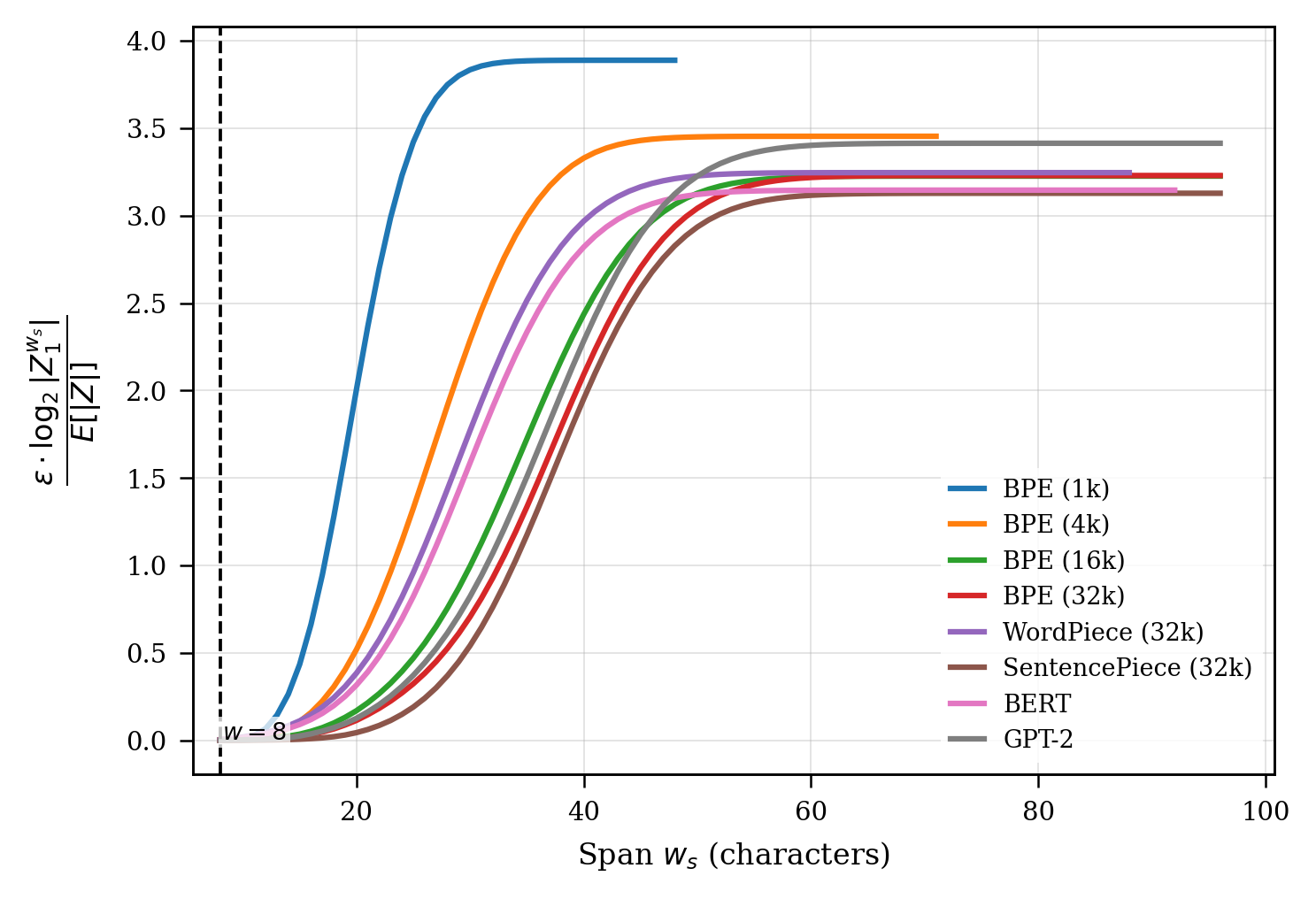}
        \caption{$w=8$}
    \end{subfigure}

    \vspace{0.4em}

    \begin{subfigure}[t]{0.245\linewidth}
        \centering
        \includegraphics[width=\linewidth]{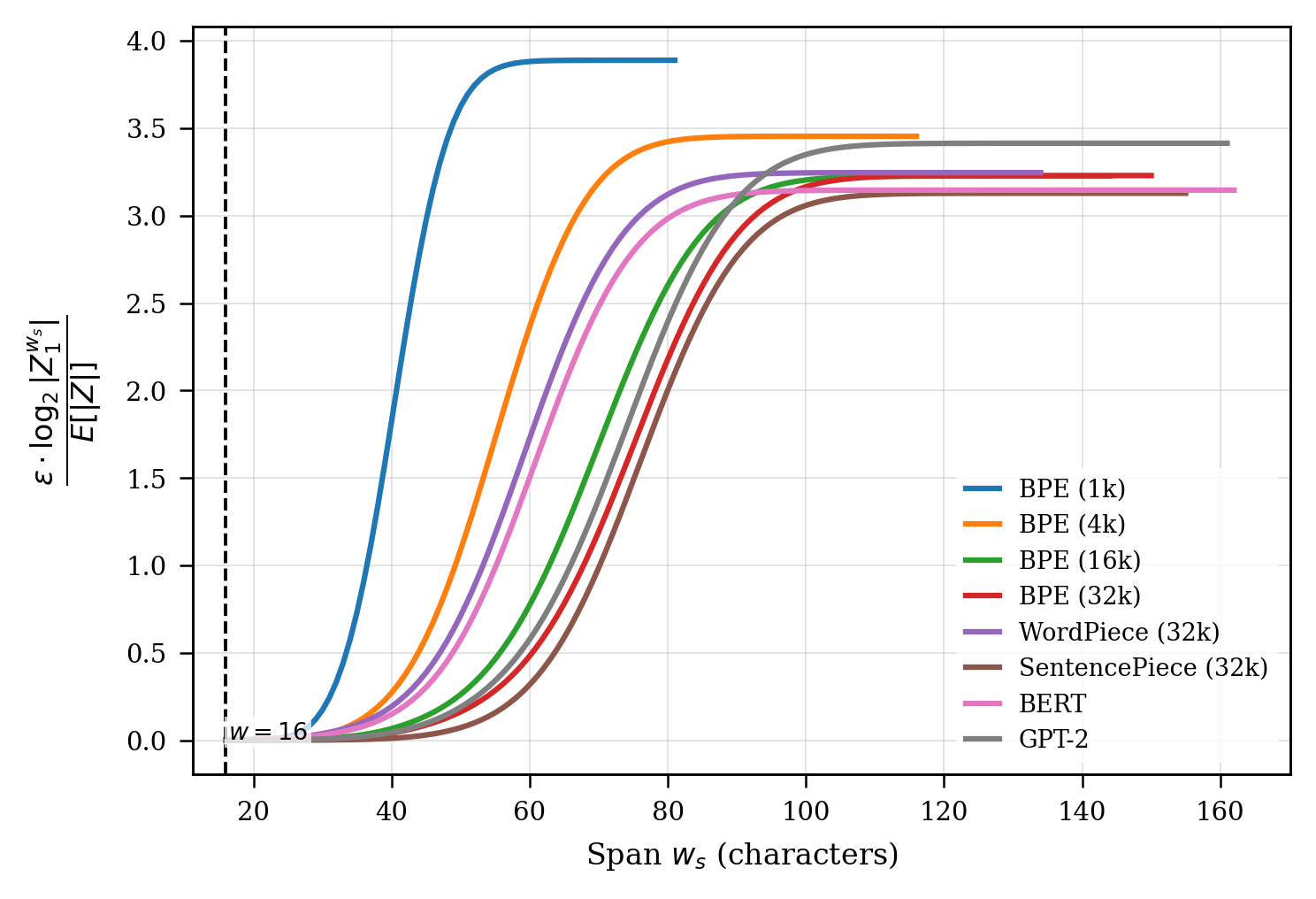}
        \caption{$w=16$}
    \end{subfigure}
    \begin{subfigure}[t]{0.245\linewidth}
        \centering
        \includegraphics[width=\linewidth]{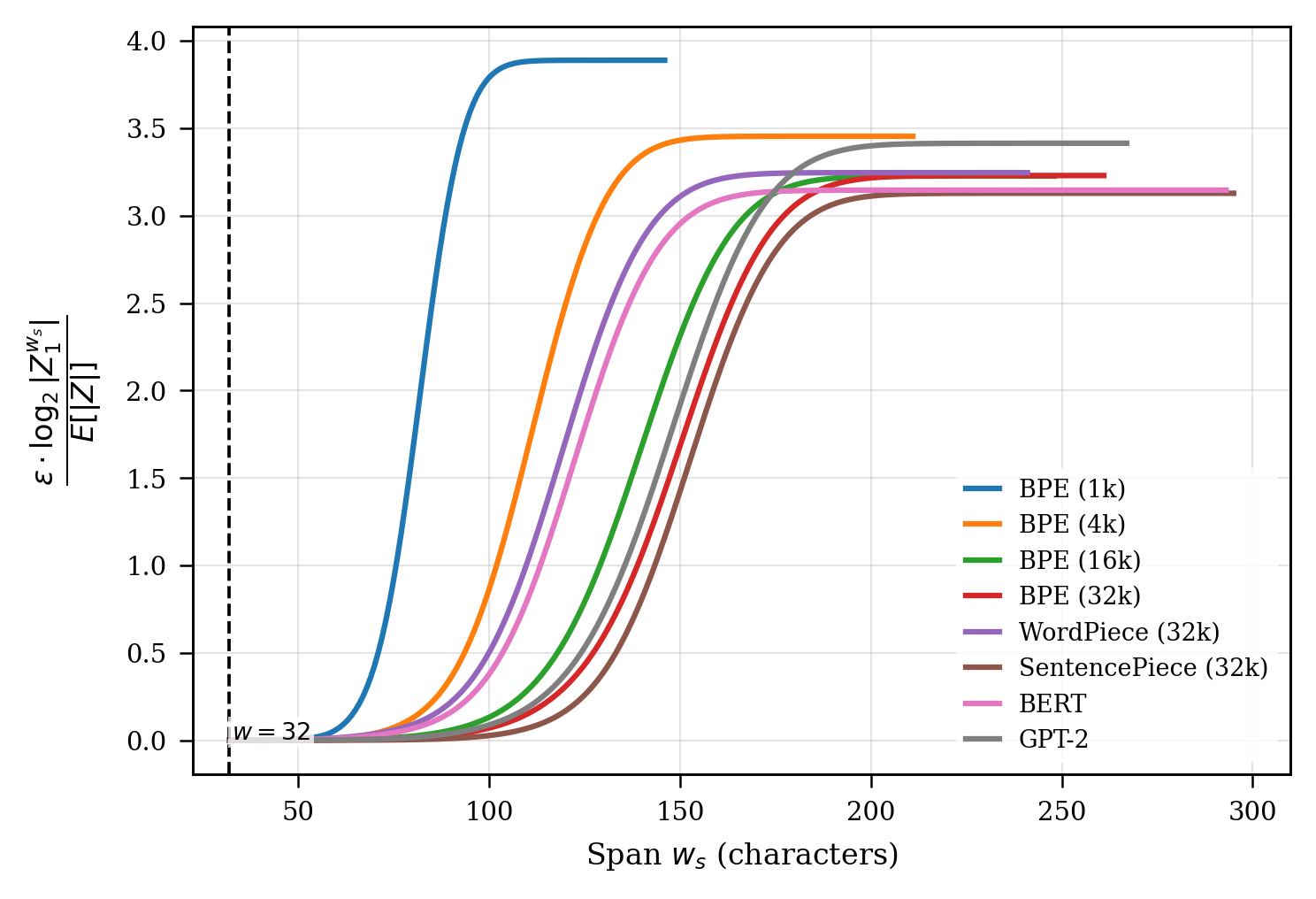}
        \caption{$w=32$}
    \end{subfigure}
    \begin{subfigure}[t]{0.245\linewidth}
        \centering
        \includegraphics[width=\linewidth]{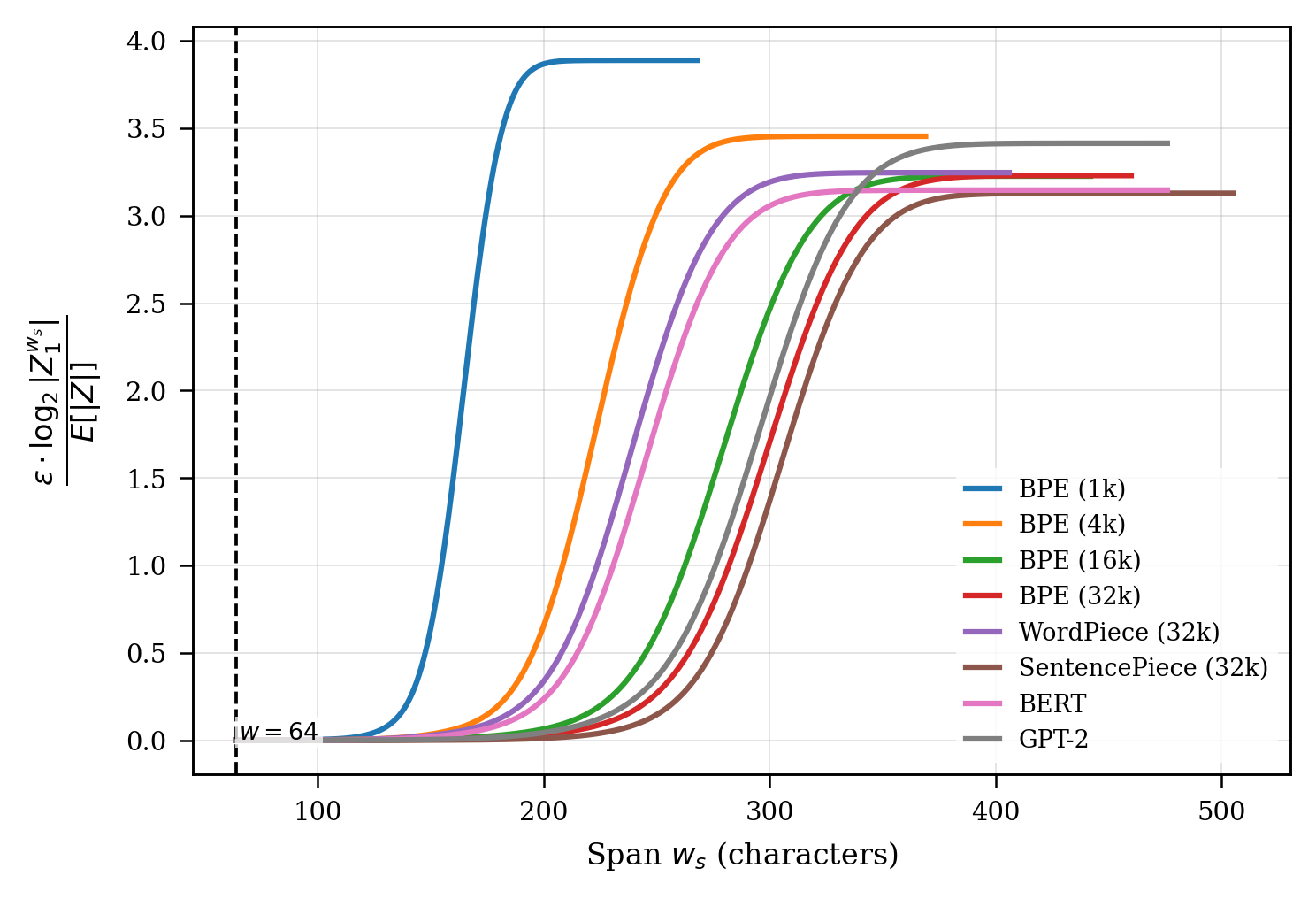}
        \caption{$w=64$}
    \end{subfigure}
    \begin{subfigure}[t]{0.245\linewidth}
        \centering
        \includegraphics[width=\linewidth]{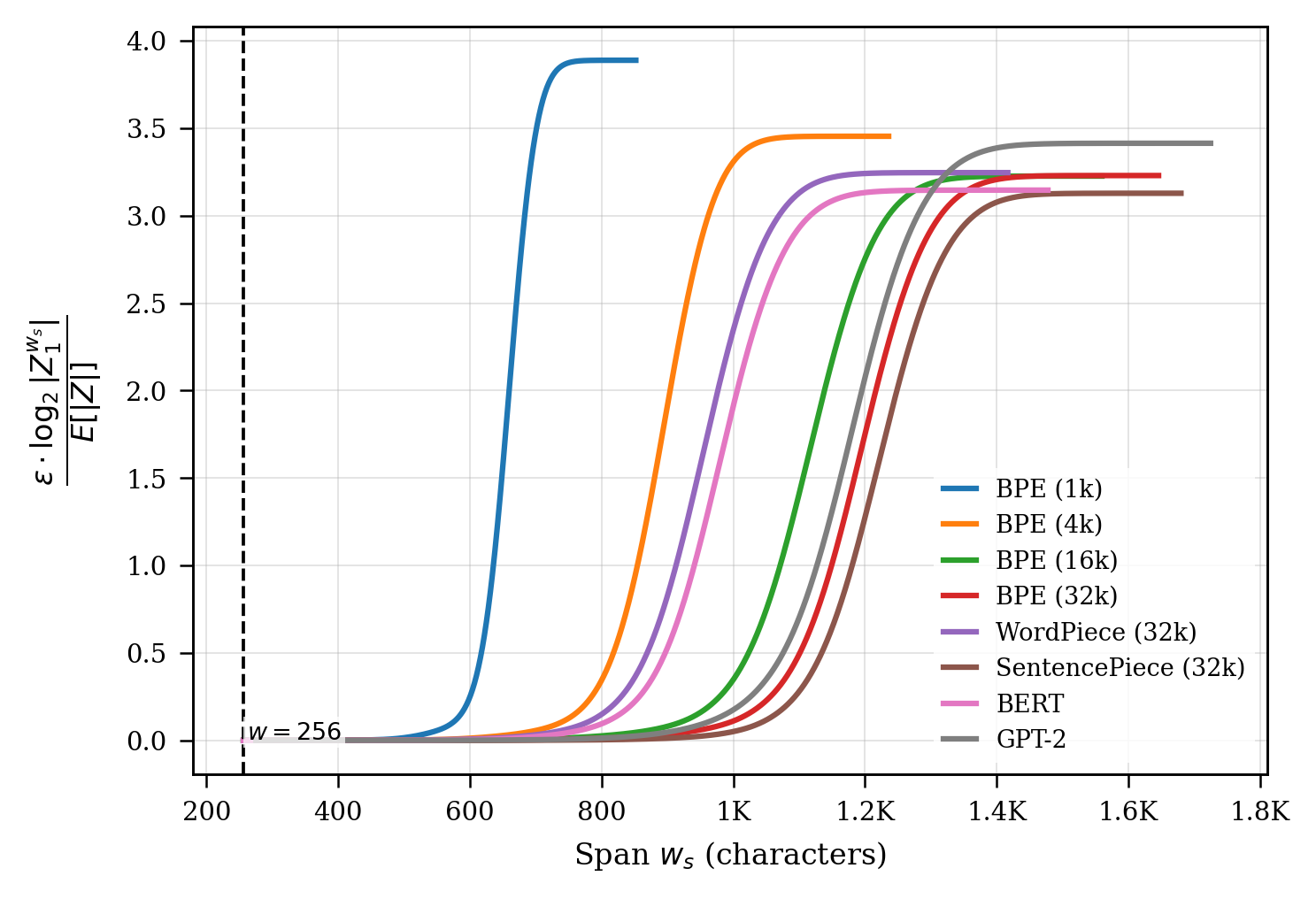}
        \caption{$w=256$}
    \end{subfigure}

    \caption{
    Slack term
    $\epsilon(w,w_s)\log_2|\mathcal Z|/\alpha_{\mathcal Z}$ from
    Theorem~\ref{thm:tokenization_gain} as a function of target source span
    $w_s$ on WikiText. Across window sizes, tokenizers allow the target source
    span to grow well beyond the token-window length before the slack becomes
    large.
    }
    \label{fig:cdf_scaled_all_windows}
\end{figure}
\newpage
\begin{figure}[t]
    \centering

    \begin{subfigure}[t]{0.245\linewidth}
        \centering
        \includegraphics[width=\linewidth]{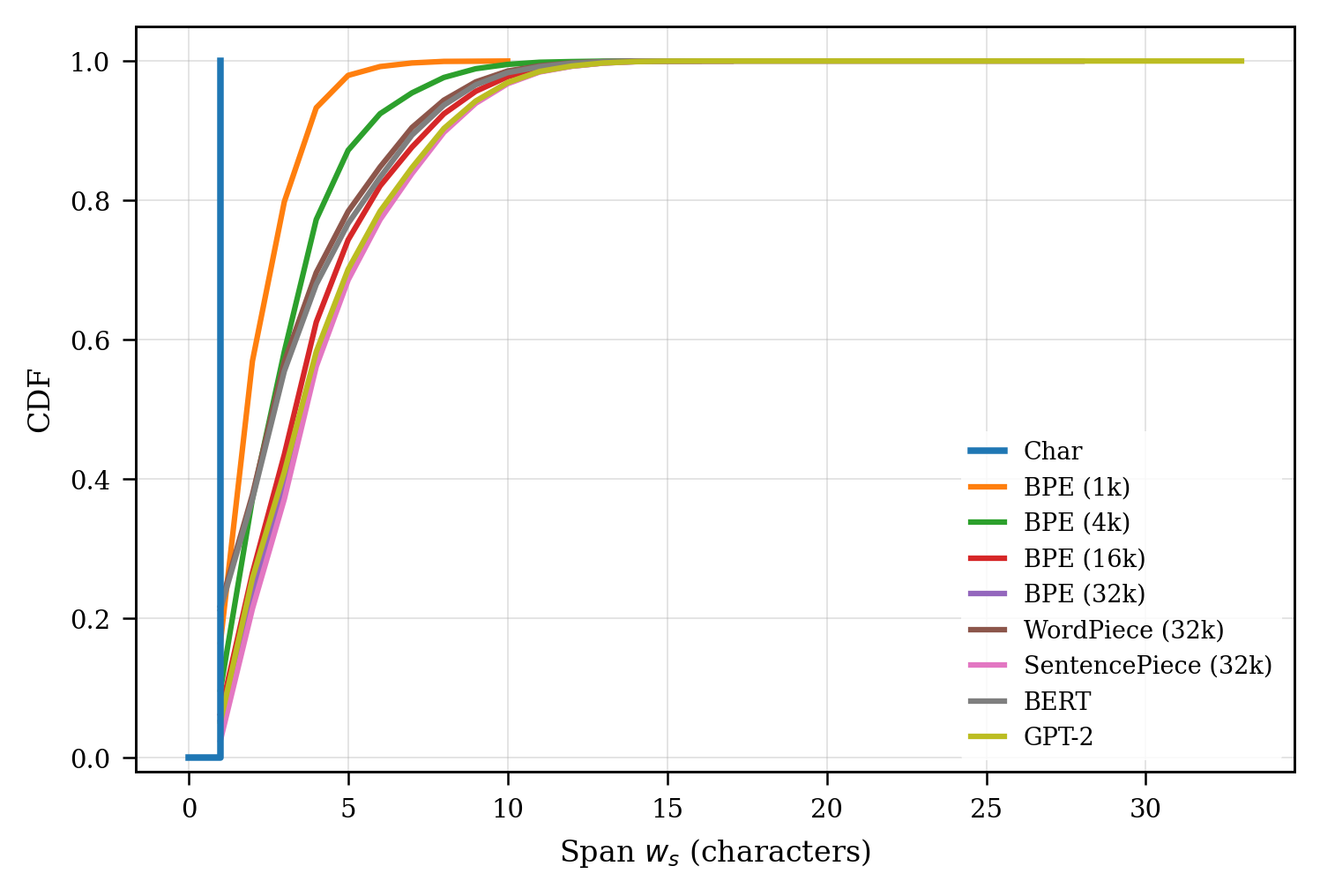}
        \caption{$w=1$}
    \end{subfigure}
    \begin{subfigure}[t]{0.245\linewidth}
        \centering
        \includegraphics[width=\linewidth]{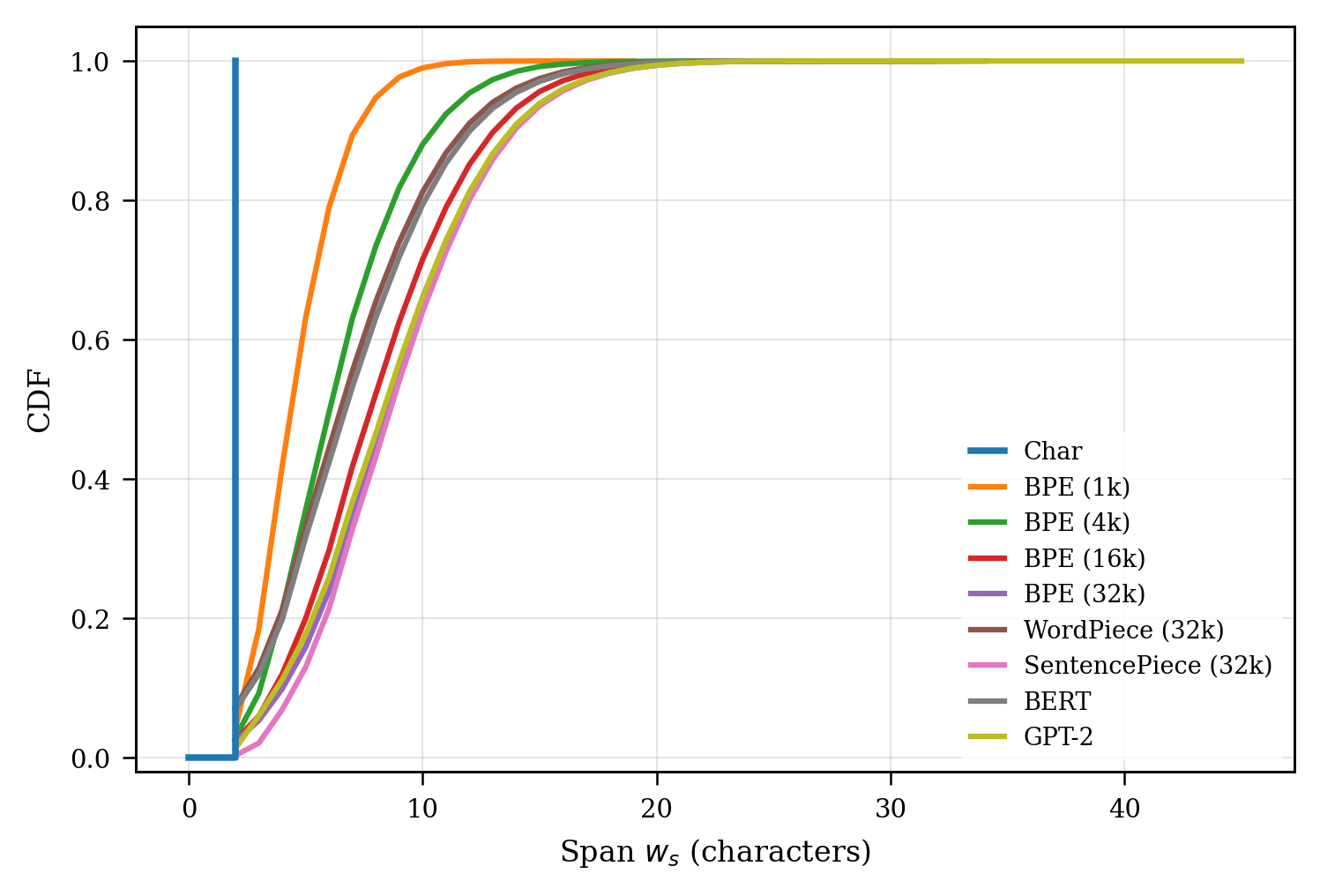}
        \caption{$w=2$}
    \end{subfigure}
    \begin{subfigure}[t]{0.245\linewidth}
        \centering
        \includegraphics[width=\linewidth]{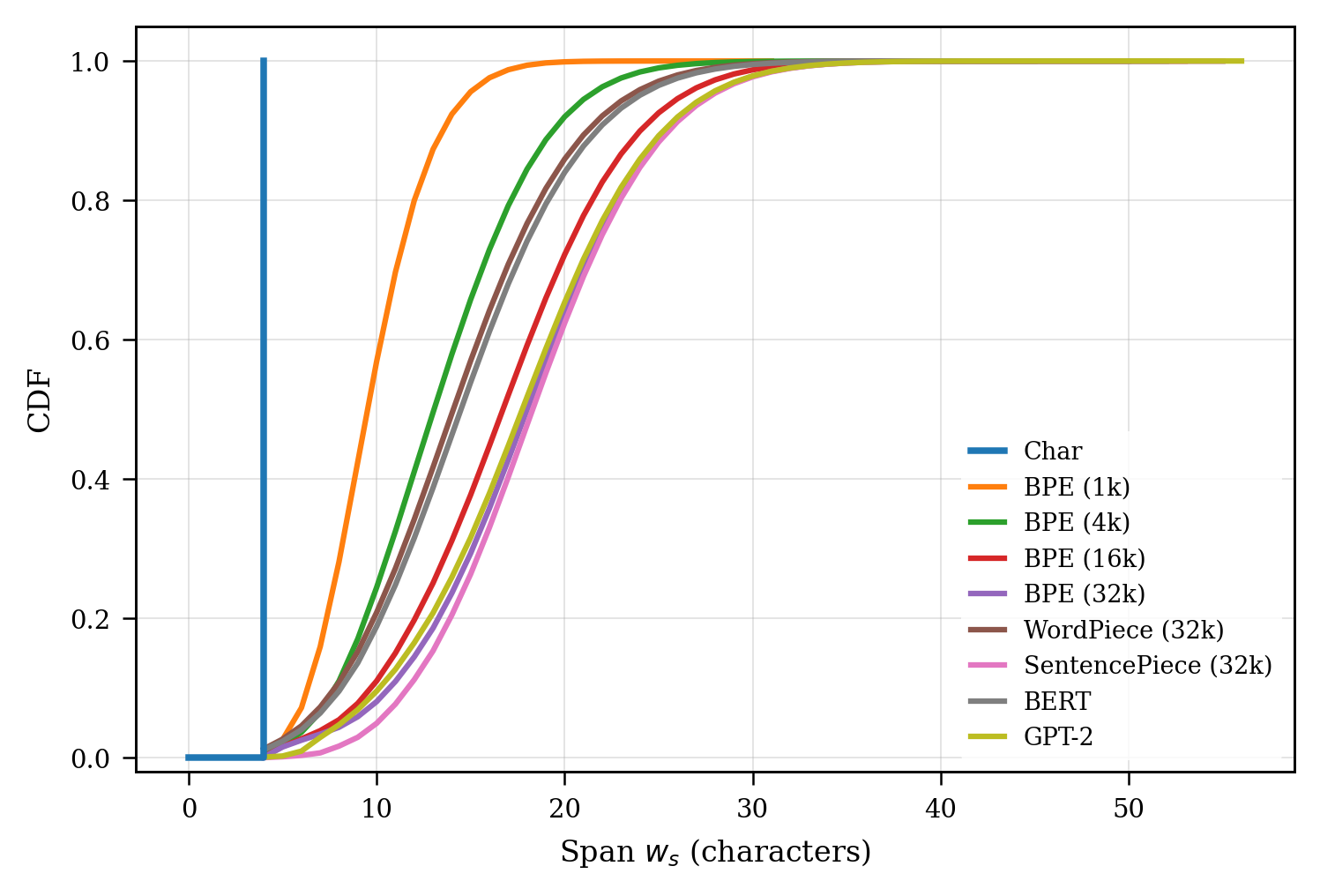}
        \caption{$w=4$}
    \end{subfigure}
    \begin{subfigure}[t]{0.245\linewidth}
        \centering
        \includegraphics[width=\linewidth]{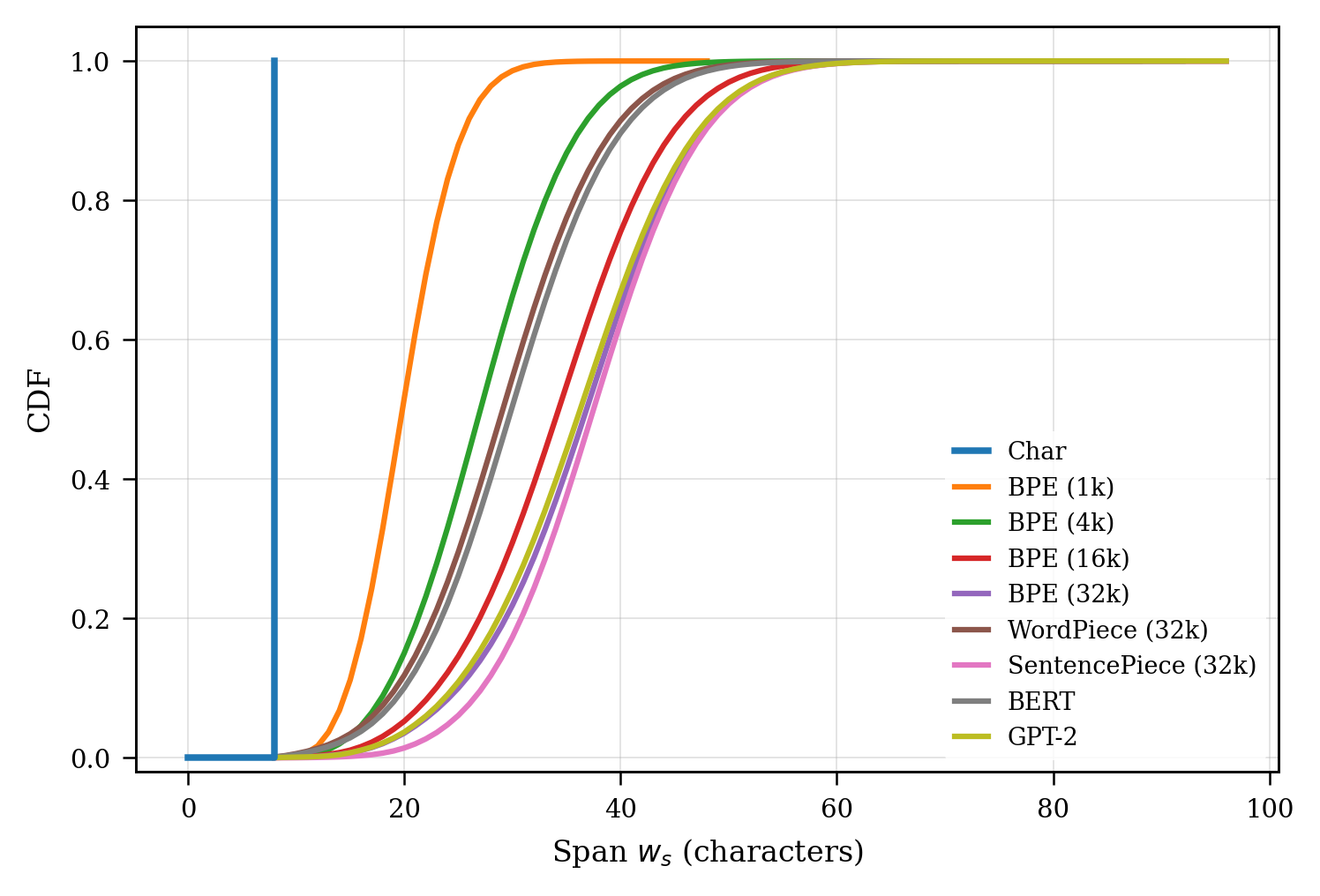}
        \caption{$w=8$}
    \end{subfigure}

    \vspace{0.4em}

    \begin{subfigure}[t]{0.245\linewidth}
        \centering
        \includegraphics[width=\linewidth]{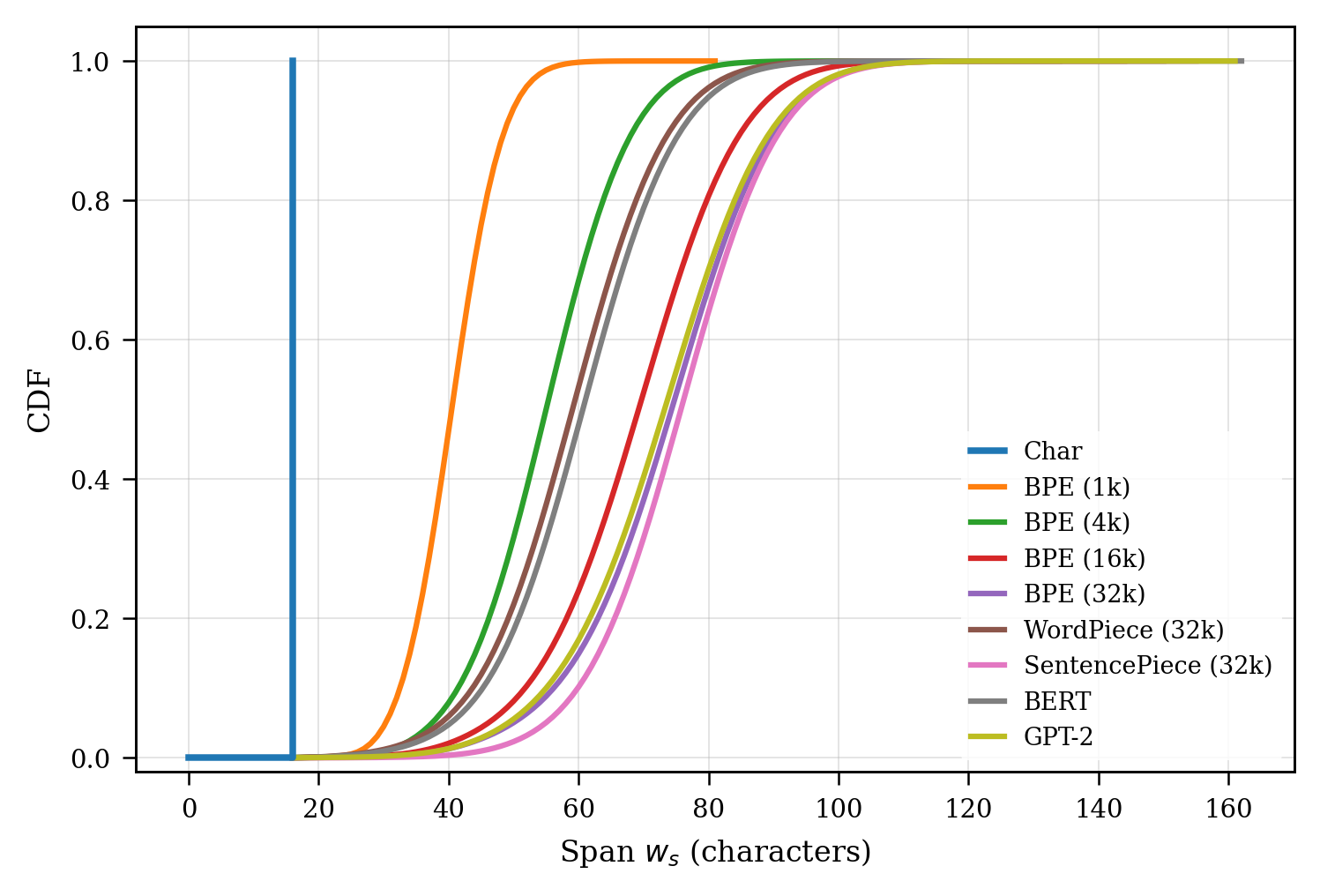}
        \caption{$w=16$}
    \end{subfigure}
    \begin{subfigure}[t]{0.245\linewidth}
        \centering
        \includegraphics[width=\linewidth]{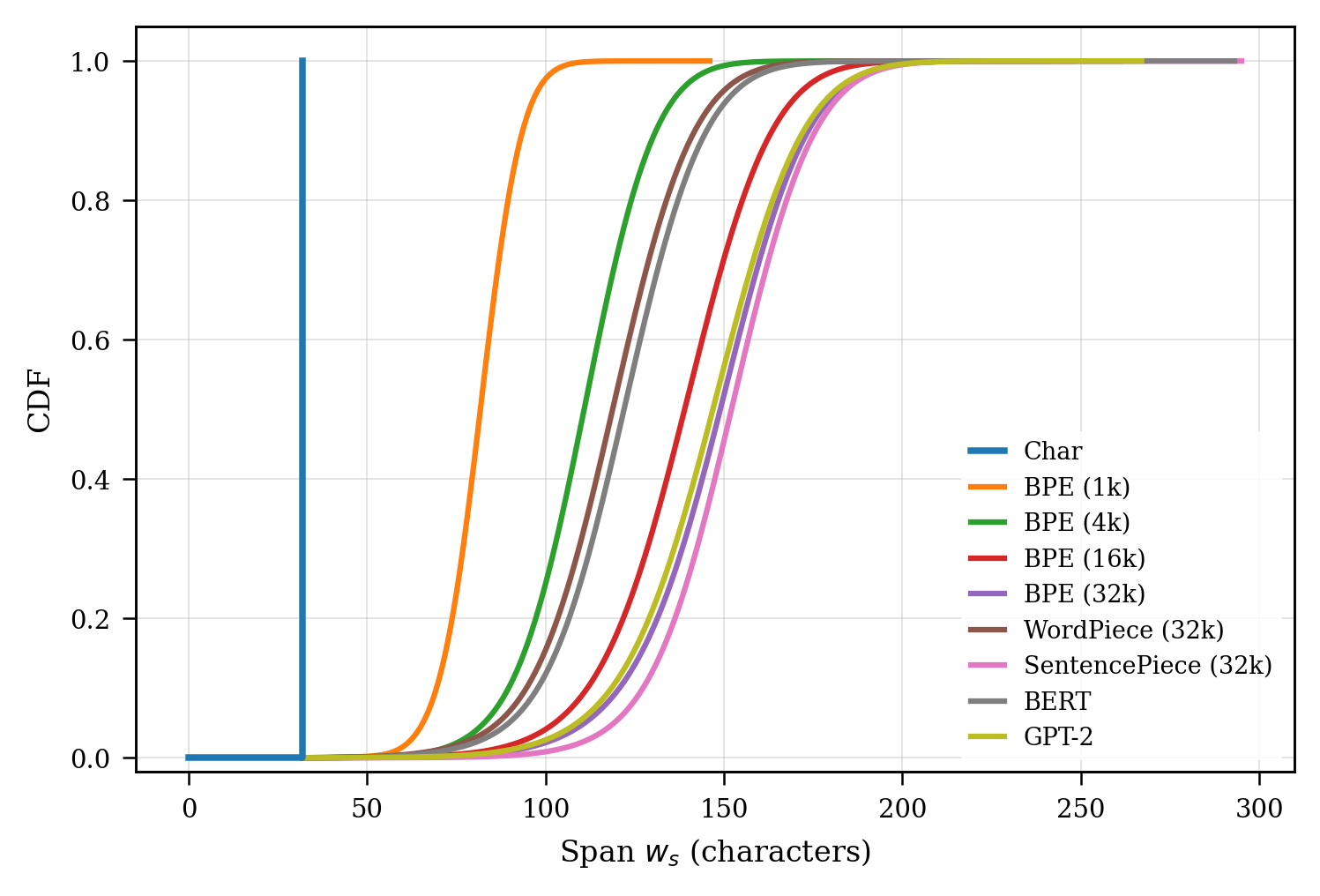}
        \caption{$w=32$}
    \end{subfigure}
    \begin{subfigure}[t]{0.245\linewidth}
        \centering
        \includegraphics[width=\linewidth]{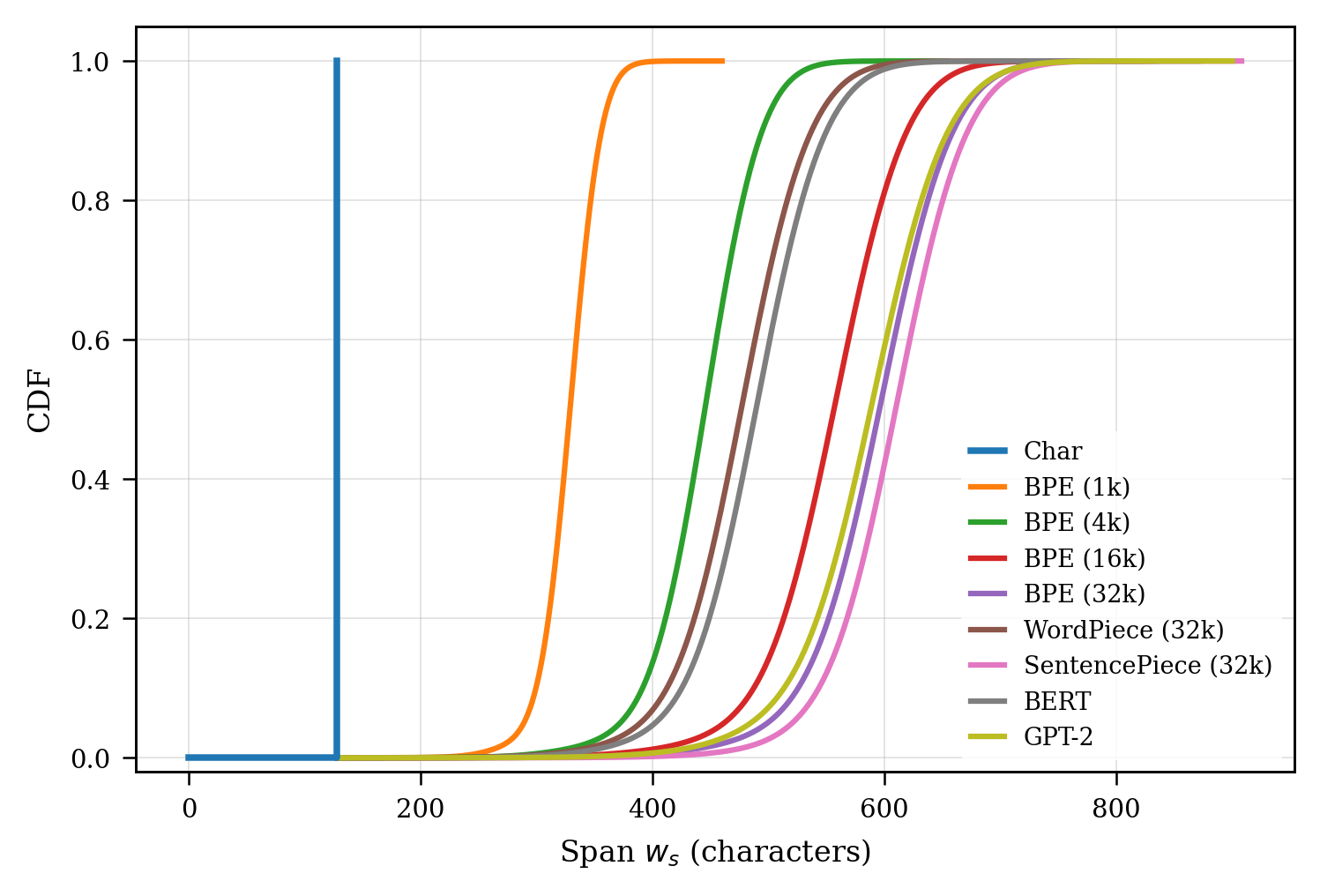}
        \caption{$w=128$}
    \end{subfigure}
    \begin{subfigure}[t]{0.245\linewidth}
        \centering
        \includegraphics[width=\linewidth]{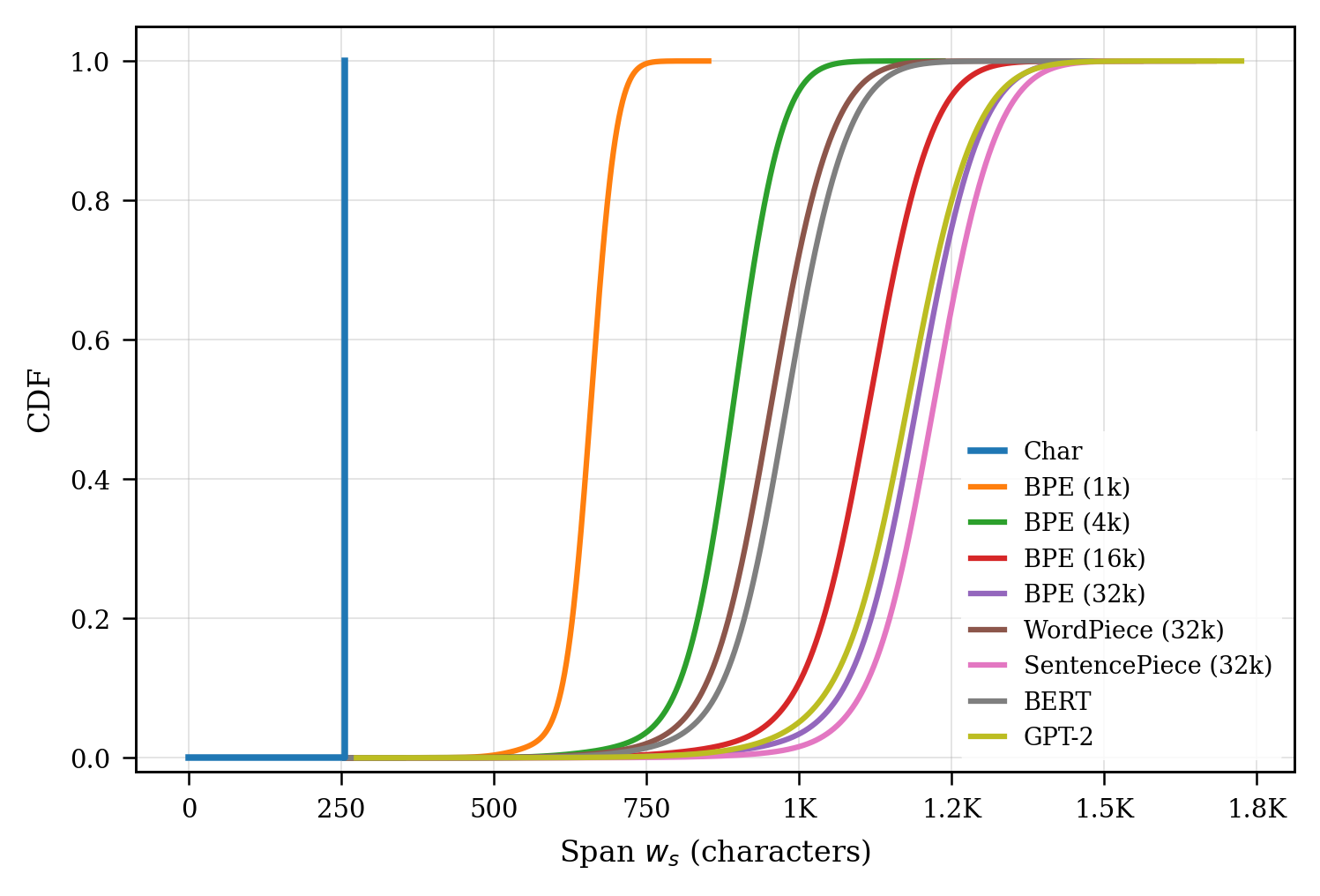}
        \caption{$w=256$}
    \end{subfigure}

    \caption{
    Empirical CDFs of the source span $S(Z_1^w)$ on WikiText for different
    token-window lengths. The character-level representation has deterministic
    span $w$, while tokenized representations shift the distribution to much
    larger source spans. This shows that the effective-context behavior observed
    in the main text is stable across window sizes.
    }
    \label{fig:cdf_all_windows}
\end{figure}

    \newpage

    \end{document}